\documentclass[12pt]{article}

\usepackage{times}
\usepackage{fullpage}
\usepackage{parskip}
\usepackage{authblk}

\usepackage{hyperref}
\usepackage{url}

\usepackage{amsfonts} 
\usepackage{amsmath}
\usepackage{amsthm, amssymb}

\usepackage{natbib}

\usepackage{algorithm}
\usepackage{algorithmic}

\usepackage{tikz}
\usepackage{tkz-graph}
\usetikzlibrary{positioning,shadows,arrows}

\makeatletter
\def\thm@space@setup{%
  \thm@preskip=\parskip \thm@postskip=0pt
}
\makeatother
\newtheorem{theorem}{Theorem}
\newtheorem{proposition}[theorem]{Proposition}
\newtheorem{lemma}[theorem]{Lemma}

\newtheorem{corollary}[theorem]{Corollary}

\newtheorem{fact}[theorem]{Fact}
\newenvironment{definition}{\goodbreak\smallskip\noindent{\bf Definition \addtocounter{theorem}{1}\arabic{theorem} }}{}

\title{On the Expressive Efficiency of Sum Product Networks}

\author{James Martens\thanks{jmartens@cs.toronto.edu} }
\author{Venkatesh Medabalimi\thanks{venkatm@cs.toronto.edu}}
\affil{Department of Computer Science, University of Toronto}
\date{}

\newcommand{\bigO}{\mathcal{O}}
\newcommand{\Real}{\mathbb{R}}

\DeclareMathOperator{\rank}{rank}
\DeclareMathOperator{\trace}{trace}
\DeclareMathOperator{\EQUAL}{EQUAL}

\let\emptyset\varnothing

\begin{document}

\maketitle



\begin{abstract}

Sum Product Networks (SPNs) are a recently developed class of deep generative models which compute their associated unnormalized density functions using a special type of arithmetic circuit.  When certain sufficient conditions, called the decomposability and completeness conditions (or ``D\&C" conditions), are imposed on the structure of these circuits, marginal densities and other useful quantities, which are typically intractable for other deep generative models, can be computed by what amounts to a single evaluation of the network (which is a property known as ``validity").  However, the effect that the D\&C conditions have on the capabilities of D\&C SPNs is not well understood.  

In this work we analyze the D\&C conditions, expose the various connections that D\&C SPNs have with multilinear arithmetic circuits, and consider the question of how well they can capture various distributions as a function of their size and depth. 
Among our various contributions is a result which establishes the existence of a relatively simple distribution with fully tractable marginal densities which cannot be efficiently captured by D\&C SPNs of \emph{any} depth, but which can be efficiently captured by various other deep generative models.  We also show that with each additional layer of depth permitted, the set of distributions which can be efficiently captured by D\&C SPNs grows in size.  This kind of ``depth hierarchy" property has been widely conjectured to hold for various deep models, but has never been proven for any of them.  Some of our other contributions include a new characterization of the D\&C conditions as sufficient \emph{and necessary} ones for a slightly strengthened notion of validity, and various state-machine characterizations of the types of computations that can be performed efficiently by D\&C SPNs.

\end{abstract}

\section{Introduction}

Sum Product Networks (SPNs) \citep{poon2011sum} are a recently developed class of deep generative models which compute their associated unnormalized density functions using a special type of arithmetic circuit.  Like neural networks, arithmetic circuits \citep[e.g.][]{shpilka2010arithmetic} are feed-forward circuits whose gates/nodes compute real values, and whose connections have associated real-valued weights.  Each node in an arithmetic circuit computes either a weighted sum or a product over their real-valued inputs.

For an important special class of SPNs called ``valid SPNs", computing the normalizing constant, along with any marginals, can be performed by what amounts to a single evaluation of the network.  This is to be contrasted with other deep generative models like Deep Boltzmann Machines \citep{DBMs}, where quantities crucial to learning and model evaluation (such as the normalizing constant) are provably intractable, unless $P = \#P$ \citep{roth-hard}.  

The tractability properties of valid SPNs are the primary reason they are interesting both from a theoretical and practical perspective.  However, validity is typically enforced via the so-called ``decomposability" and ``completeness" conditions (which we will abbreviate as ``D\&C"). While easy to describe and verify, the D\&C conditions impose stringent structural restrictions on SPNs which limit the kinds of architectures that are allowed.  While some learning algorithms have been developed that can respect these conditions \citep[e.g.][]{SPN_Learn_Gens,SPN_Learn_Peharz,SPN_Learn_Rooshenas}, the extent to which they limit the expressive efficiency\footnote{By this we mean the extent to which they can efficiently capture various distribution.  A distribution is ``efficiently captured" if it is contained in the closure of the set of distributions corresponding to different settings of the models parameters, for polynomial sized (in the dimension $n$ of the data/input) instances of the model, where size is measured by the number of ``units" or parameters.  Often these will be realistic low-order polynomials, although this depends on how exactly the constructions are done.  Note that a distribution being ``efficiently captured" says nothing about how easily the marginal densities or partition function of its associated density can be computed (\emph{except} in the case of D\&C SPNs of course). 

The concept of expressive efficiency is also sometimes called ``expressive power" or ``representational power", although we will use the word ``efficiency" instead of ``power" to emphasize our focus on the question of whether or not certain distributions which can be captured \emph{efficiently} by the model, instead of the question of whether or not they can be captured at all (i.e. by super-polynomially sized instances of the model).  This latter question is the topic of papers which present so-called ``universality" results which show how some models can capture \emph{any} distribution if they are allowed be exponentially large in $n$ (by essentially simulating a giant look-up table).  Such results are fairly straightforward, and indeed it easy to show that D\&C SPNs are universal in this sense. } of SPNs versus various other deep generative models remains unclear.

Like most models, D\&C SPNs are ``universal" in the sense that they can capture any distribution if they are allowed to be of a size which is exponential in the dimension $n$ of the data/input. However, any distribution function which can be \emph{efficiently} captured by D\&C SPNs, which is to say by one of polynomial size, must therefore be tractable (in the sense of having computable marginals, etc).  And given complexity theoretic assumptions like $P \neq \#P$ it is easy to come up with density functions whose marginals/normalizers are intractable, but which nonetheless correspond to distributions which can be efficiently captured by various other deep generative models (e.g. using the simulation results from \citet{Deep_Expressive_Martens}).   Thus we see that the tractability properties enjoyed by D\&C SPNs indeed come with a price. 

However, one could argue that the intractability of these kinds of ``hard" distributions would make it difficult or even impossible to learn them in practice.  Moreover, any model which can efficiently capture them must therefor lack an efficient general-case inference/learning algorithm.  This is a valid point, and it suggests the obvious follow-up question: is there a \emph{fully tractable} distribution (in the sense that its marginal densities and partition function can be computed efficiently) which can be efficiently captured by other deep models, but not by D\&C SPNs?  

In this work we answer this question in the affirmative, without assuming any complexity theoretic conjectures.  This result thus establishes that D\&C SPNs are in some sense less expressively efficient than many other deep models (since, by the results of \citet{Deep_Expressive_Martens}, such models can efficiently simulate D\&C SPNs), even if we restrict our attention only to tractable distributions.   Moreover, it suggests existence of a hypothetical model which could share the tractability properties of D\&C SPNs, while being more expressively efficient.

In addition to this result, we also analyze the effect of depth, and other structural characteristics, on the expressive efficiency of D\&C SPNs.  Perhaps most notably, we use existing results from arithmetic circuit theory to establish that D\&C SPNs gain expressive efficiency with each additional layer of depth.  In particular, we show that the set of distributions which can be efficiently captured by D\&C SPNs grows with each layer of depth permitted.  This kind of ``depth hierarchy" property has never before been shown to hold for any other well-known deep model, despite the widespread belief that it does hold for most of them \citep[e.g.][]{bengio2011expressive}.



Along with these two results, we also make numerous other contributions to the theoretical understanding of SPNs which are summarized below.


In Section \ref{sec:definitions}, we first propose a generalized definition of SPNs that captures all previous definitions. We then illuminate the various connections between SPNs and multilinear arithmetic circuits, allowing us to exploit the many powerful results which have already been proved for the latter.

In Section \ref{sec:analysis} we provide new insights regarding the D\&C conditions and their relationship to validity, and introduce a slightly strengthened version of validity which we show to be equivalent to the D\&C conditions (whereas standard validity is merely \emph{implied} by them).  We also show that for a slightly generalized definition of SPNs, testing for standard validity is a co-NP hard problem.

In Section \ref{sec:capabilities} we give examples of various state-based models of computation which can be efficiently simulated by D\&C SPNs, and show how these can be used to give constructive proofs that various simple density functions can be efficiently computed by D\&C SPNs.

In Section \ref{sec:depth3separation} we address the prior work on the expressive efficiency of D\&C SPNs due to \citet{DelalleauBengio2011}, and give a much shorter proof of their results using powerful techniques borrowed from circuit theory.  We go on to show how these techniques allow us to significantly strengthen and extend the results of \citet{DelalleauBengio2011}, answering an open question which they posed.

In Section \ref{sec:depth_analysis} we leverage prior work done on multilinear arithmetic circuits to prove several very powerful results regarding the relationship between depth and expressive efficiency of D\&C SPNs.  First, we show that with each extra layer of depth added, there is an expansion of the set of functions efficiently computable by D\&C SPNs (thus giving a strict ``hierarchy of depth").  Next we show that if depth is allowed to grow with the input dimension $n$, that its effect on expressive efficiency greatly diminishes after it reaches $\bigO(\log(n)^2)$.

In Section \ref{sec:formula_limitation} we show that when D\&C SPNs are constrained to have a recursive ``formula" structure, as they are when learned using the approach of \citep{SPN_Learn_Gens}, they lose expressive efficiency.  In particular we use prior work on multilinear arithmetic circuits to produce example functions which can be efficiently computed by general D\&C SPNs, but not by ones constrained to have a formula structure.

Finally, in Section \ref{sec:spanning_tree_lb} we give what is perhaps our most significant and difficult result, which is the existence of a simple density function whose marginals and normalizer are computable by an $\bigO(n^{1.19})$ time algorithm, and whose corresponding distribution can be efficiently captured by various other deep models (in terms of their size), but which cannot be efficiently computed, or even efficiently approximated, by a D\&C SPN of \emph{any} depth.

\section{Definitions and Notation}

\label{sec:definitions}

\subsection{Arithmetic circuits}
\label{sec:arithmetic_circuit_defs}

Arithmetic circuits \citep[e.g.][]{shpilka2010arithmetic} are a type of circuit, similar to Boolean logic circuits, or neural networks. But instead of having gates/nodes which compute basic logical operations like AND, or sigmoidal non-linearities, they have nodes which perform one of the two fundamental operations of arithmetic: addition and multiplication.  Their formal definition follows.

An \emph{arithmetic circuit} $\Phi$ over a set/tuple\footnote{By ``set/tuple" we mean a tuple like $t = (a,b,c)$ which we will occasionally treat like a standard set, so that expressions such as $t \cap s$ are well defined and have the natural interpretation.} of real-valued variables $y = (y_1,y_2,..,y_\ell)$ will be defined as a special type of directed acyclic graph with the following properties.  

Each node of the graph with in-degree 0 is labeled by either a variable from $y$ or an element from $\Real$. Every other node is labeled by either $\times$ or $+$, and are known as product nodes or sum nodes respectively.  All the incoming edges to a sum node are labeled with weights from $\Real$.
Nodes with no outgoing edges are referred to as \emph{output nodes}.  We will assume that arithmetic circuits only have one output node, which we will refer to as the \emph{root}.  Nodes with edges going into a node $u$ in $\Phi$ are referred to as $u$'s \emph{children}.  The set of such children is denoted by $C(u)$.

Given values of the elements of $y$, a node $u$ of an arithmetic circuit computes a real-valued output, which we denote by $q_u(y)$, according to the following rules.  When $u$ is labeled with an element of $y$ or $\Real$, the node simply computes its label.  The latter type of nodes are referred to as \emph{constant nodes}, since they compute constants that don't depend on $y$.  Product nodes compute the product of the outputs of their children, i.e. $q_u(y) = \prod_{v \in C(u)} q_v(y)$, while sum nodes compute a weighted sum of the outputs of their children, i.e. $q_u(y) = \sum_{v \in C(u)} w_{v,u} q_v(y)$, where $w_{u,v}$ denotes the weight labeling the edge from from $v$ to $u$.  Given these definitions, it is not hard to see that for each node $u$ of an arithmetic circuit, $q_u(y)$ is a multivariate polynomial function in the elements of $y$.  The output of $\Phi$, denoted by $q_{\Phi}(y)$, is defined as the output of its singular root/output node (i.e. $q_r(y)$, where $r$ is the root/output node of $\Phi$).

For a node $u$ in $\Phi$, $\Phi_u$ denotes the subcircuit of $\Phi$ rooted at $u$.  This subcircuit is formed by taking only the nodes in $\Phi$ that are on a path to $u$.

An arithmetic circuit is said to be \emph{monotone} if all of its weights and constants are non-negative elements of $\Real$.  

The \emph{scope} of a node $u$, denoted by $y_u$, is defined as the subset of the elements of $y$ which appear as labels in the sub-circuit rooted at $u$.  These are the variables which $u$'s output essentially ``depends on".

The \emph{size} of $\Phi$, denoted by $|\Phi|$, is defined as the number of nodes in $\Phi$, and its \emph{depth} is defined as the length of the longest directed path in $\Phi$.  An alternative notion of depth, called \emph{product depth} \citep{raz2009lower}, is defined as the largest number of product nodes which appear in a directed path in $\Phi$.

Note that in general, nodes in an arithmetic circuit can have out-degree greater than 1, thus allowing the quantities they compute to be used in multiple subsequent computations by other nodes.  When this is not the case and the nodes of $\Phi$ each have out-degree at most 1, $\Phi$ is said to be an \emph{arithmetic formula}, because it can be written out compactly as a formula.  


\subsection{Sum Product Networks (SPNs)}

\label{sec:def_SPNs}

In this section we will give our generalized definition of Sum Product Networks (SPNs). 

Let $x = (x_1,x_2,..,x_n)$ be a set/tuple of variables where $x_i$ can take values in a range set $R_i \subseteq \Real$ and $M_1, M_2, ..., M_n$ are measures over the respective $R_i$'s.  For each $i$ let \\
$f_i = (f_{i,1}(x_i),f_{i,2}(x_i),...,f_{i,m_i}(x_i))$ denote a set/tuple of $m_i$ \emph{non-negative} real-valued univariate functions of $x_i$, each with a finite $M_i$-integral over $R_i$.  $f$ will denote the set/tuple whose elements are given by appending all of these $f_i$'s together, and $M$ will denote the product measure $M_1 \times M_2 \times ... \times M_n$.

A \emph{Sum Product Network} (SPN) $\Phi$ is defined as a monotone arithmetic circuit over $f$.  It inherits all of the properties of monotone arithmetic circuits, and gains some additional ones, which are discussed below. 


Because an SPN is an arithmetic circuit over $f$, any one of its nodes $u$ computes a polynomial function $q_u(f)$ in $f$.   But because the elements of $f$ are functions of the elements of $x$, a node $u$ of an SPN can be also viewed as computing a function of $x$, as given by $q_u(f(x))$, where $f(x)$ denotes the set/tuple obtained by replacing each element $f_{i,j}$ in $f$ with the value of $f_{i,j}(x_i)$.

The \emph{dependency-scope} of a node $u$ is defined as the set of elements of $x$ on which members of $u$'s scope $f_u$ depend.  The dependency-scope is denoted by $x_u$.

SPNs are primarily used to model distributions over $x$.  They do this by defining a density function given by $p_{\Phi}(x) = \frac{1}{Z} q_{\Phi}(f(x))$, where $Z = \int q_{\Phi}(f(x)) \mathrm{d} M(x)$ is a normalizing constant known as the partition function.  Because $q_{\Phi}(f(x))$ is non-negative this density is well defined, provided that $Z$ is non-zero and finite.

A \emph{formula SPN} is defined as an SPN which is also an arithmetic formula.  In other words, a formula SPN is one whose nodes each have an out-degree of at most 1.

It is important to remember that the domains (the $R_i$'s) and the measures ($M_i$'s) can be defined however we want, so that SPNs can represent both continuous and discrete distributions.  For example, to represent a discrete distribution, we can choose $M_i$ to be the counting measure with support given by a finite subset, such as $\{0,1\}$.  In such a case, integration of some function $g(x_i)$ w.r.t. such a $M_i$ amounts to the summation $\sum_{x_i \in \{0,1\}} g(x_i)$.

\subsection{Validity, decomposability, and completeness}

Treated as density models, general SPNs suffer from many of the same intractability issues that plague other deep density models, such as Deep Boltzmann Machines \citep{DBMs}.   In particular, there is no efficient general algorithm for computing their associated partition function $Z$ or marginal densities.

However, it turns out that for a special class of SPNs, called \emph{valid} SPNs, computing the the partition function and marginal densities can be accomplished by what is essentially a single evaluation of the network.  Moreover, the validity of a given SPN can be established using certain easy-to-test structural conditions called decomposability and completeness, which we will discuss later.

\begin{definition}(Valid SPN) 
\label{def:valid_SPN}
An SPN $\Phi$, is said to be \emph{valid} if the following condition always holds. Let $I = (i_1, i_2, ..., i_\ell)$ where each $i_j$ is a distinct element of $[n] = \{1,2,...,n\}$, and let $S_{i_1} \subseteq R_{i_1}, S_{i_2} \subseteq R_{i_2}, ..., S_{i_\ell} \subseteq R_{i_\ell}$ be subsets of the ranges of the respective $x_{i_j}$'s.  For any fixed value of $x_{[n] \setminus I}$ we have
\begin{align*}
\int_{S_{i_1} \times S_{i_2} \times ... \times S_{i_\ell}} q_{\Phi}(f(x)) \: \mathrm{d} M_I(x_I) = q_{\Phi}(A_I( S_I, x_{[n] \setminus I}))
\end{align*}
where $x_I = (x_{i_1}, x_{i_2}, ..., x_{i_\ell})$ (with $x_{[n] \setminus I}$ defined analogously), $M_I = M_{i_1} \times M_{i_2} \times ... \times M_{i_\ell}$ (with $S_I$ defined analogously), and where $A_I( S_I, x_{[n] \setminus I})$ denotes the set/tuple obtained by taking $f$ and for each $i \in I$ and each $j$ replacing $f_{i,j}$ with its integral $\int_{S_i} f_{i,j}(x_i) \mathrm{d} M_i(x_i)$ over $S_i$, and also replacing $f_{i,j}$ for each $i \in [n] \setminus I$ with $f_{i,j}(x_i)$.
\end{definition}

Decoding the notation, this definition says that for a valid SPN $\Phi$ we can compute the integral of the output function $q_{\Phi}(f(x))$ with respect to a subset of the input variables (given by the index set $I$) over corresponding subsets of their respective domains (the $S_i$'s), simply by computing the corresponding integrals over the respective univariate functions (the $f_{i,j}$'s) and evaluating the circuit by having nodes labeled by these $f_{i,j}$'s compute said integrals.   

Note that for a subsets of the range of $R_{i_1} \times R_{i_2} \times ... \times R_{i_\ell}$ of $x_I$ that do not have the form of a Cartesian product $S_{i_1} \times S_{i_2} \times ... \times S_{i_\ell}$, validity doesn't say anything.  In general, the integral over such a set will be intractable for valid SPNs.

Validity is a very useful property for an SPN $\Phi$ to have, as it allows us to efficiently integrate $q_{\Phi}(f(x))$ with respect to any subset of variables/elements of $x$ by performing what amounts to a single evaluation of $\Phi$.  Among other uses \citep{SPN_Learn_Gens}, this allows us to efficiently compute the partition function\footnote{As an aside, note that validity also acts as a proof of the finiteness of the partition function, provided each integral $\int_{R_i} f_{i,j}(x_i) \mathrm{d} M(x_i)$ is finite.} which normalizes $\Phi$'s associated density function $p_{\Phi}$ by taking $I = [n]$ and $S_i = R_i$ for each $i \in [n]$.  It also allows us to efficiently compute any marginal density function by taking $I \subset [n]$ and $S_i = R_i$ for each $i \in I$.

While validity may seem like a magical property for an SPN to have, as shown by \citet{poon2011sum} there is a pair of easy to enforce (and verify) structural properties which, when they appear together, imply validity.  These are known as ``decomposability" and ``completeness", and are defined as follows.

\begin{definition} (Decomposable)
An SPN $\Phi$ is \emph{decomposable} if for every product node $u$ in $\Phi$ the dependency-scopes of its children are pairwise disjoint. 
\end{definition}

\begin{definition} (Completeness)
An SPN $\Phi$ is complete if for every sum node $u$ in $\Phi$ the dependency-scopes of its children are all the same.  
\end{definition}

As was the case in the work of \citet{poon2011sum}, decomposability and completeness turn out to be sufficient conditions, but not necessary ones, for ensuring validity according to our more general set of definitions.   Moreover, we will show that for a natural strengthening of the concept of validity, decomposability and completeness become \emph{necessary} conditions as well.

The tractability of the partition function and marginal densities is a virtually unheard of property for deep probabilistic models, is the primary reason that decomposably and complete SPNs are so appealing.

For the sake of brevity we will call an SPN which satisfies the decomposability and completeness conditions a \emph{D\&C SPN}.

A notion related to decomposability which was discussed in \citet{poon2011sum} is that of ``consistency", which is defined only for SPNs whose univariate functions $f$ are either the identity function $g(z) = z$ or the negation function $g(z) = 1-z$, and whose inputs variables $x$ are all 0/1-valued.  Such an SPN is said to be consistent if each product node satisfies the property that if one of its children has the identity function of $x_i$ in its scope, then none of the other children can have the negation function of $x_i$ in their scopes.  This is a weaker condition than decomposability, and is also known to imply validity \citep{poon2011sum}.

Note that for $0/1$-valued variables we have $x_i^2 = x_i$ and $(1-x_i)^2 = 1-x_i$, and so it is possible to construct an equivalent decomposable SPN from a consistent SPN by modifying the children of each product node so as to remove the ``redundant" factors of $x_i$ (or $1 - x_i$).  Note that such a construction may require the introduction of polynomially many additional nodes, as in the proof of Proposition \ref{prop:completeness_transform}.  In light of this, and the fact that consistency only applies to a narrowly defined sub-class of SPNs, we can conclude that consistency is not a particularly interesting property to study by itself, and so we will not discuss it any further.

\subsection{Top-down view of D\&C SPNs}

For a D\&C SPN $\Phi$ it is known (and is straightforward to show) that if the weights on the incoming edges to each sum node sum to 1, and the univariate functions have integrals of 1 (i.e. so they are normalized density functions), then the normalizing constant of $\Phi$'s associated density is 1, and each node can be interpreted as computing a normalized density over the variables in its dependency scope.   We will call such a $\Phi$ ``weight-normalized".

A weight-normalized D\&C SPN $\Phi$ can be interpreted as a top-down directed generative model where each sum node corresponds to a mixture distribution over the distributions associated with its children (with mixture weights given by the corresponding edge weights), and where each product node corresponds to factorized distribution, with factors given by the distributions of its children \citep{SPN_Learn_Gens}.  Given this interpretation it is not hard to see that sampling from $\Phi$ can be accomplished in a top-down fashion starting at the root, just like in a standard directed acyclic graphical model.

One interesting observation we can make is that it is always possible to transform a general D\&C SPN into an equivalent weight-normalized one, as is formalized in the following proposition:

\begin{proposition}
\label{prop:normalize}
Given a D\&C SPN $\Phi$ there exists a weight-normalized D\&C SPN $\Phi'$ with the same structure as $\Phi$ and with the same associated distribution.
\end{proposition}

\vspace{0.1in} 

\subsection{Relationship to previous definitions}

Our definitions of SPNs and related concepts subsume those given by \citet{poon2011sum} and later by \citet{SPN_Learn_Gens}.  Thus the various results we prove in this paper will still be valid according to those older definitions.  

The purpose of this subsection is justify the above claim, with a brief discussion which assumes pre-existing familiarity with the definitions given in the above cited works.

First, to see that our definition of SPNs generalizes that of \citet{poon2011sum}, observe that we can take the univariate functions to be of the form $x_i$ or $\overline{x_i} = 1 - x_i$, and that we can choose the domains of measures so that the $x_i$'s are discrete $\{0,1\}$-valued variables, and choose the associated measures so that integration over values of $x_i$ becomes equivalent to summation.

Second, to see that our definition generalizes that of \citet{SPN_Learn_Gens}, observe that we can take the univariate functions to be univariate density functions. 
And while \citet{SPN_Learn_Gens} formally defined SPNs as always being decomposable and complete, we will keep the concepts of SPNs and D\&C SPNs separate in our discussions, as in the original paper by \citet{poon2011sum}.


\subsection{Polynomials and multilinearity}

\label{sec:polys_and_multi}

In this section we will define some additional basic concepts which will be useful in our analysis of SPNs in the coming sections.

Given a set/tuple of formal variables $y = (y_1, y_2, ..., y_\ell )$, a \emph{monomial} is defined as a product of elements of $y$ (allowing repeats). For example, $y_1 y_2^3$ is a monomial.  In an arithmetic expression, a ``monomial term" refers to a monomial times a coefficient from $\Real$.  For example $4 y_1 y_2^4$ is a monomial term in $2y_1 + 4 y_1 y_2^4$.   

In general, polynomials over $y$ are defined as a finite sum of monomial terms.  Given a monomial $m$, its \emph{associated coefficient} in a polynomial $q$ will refer to the coefficient of the monomial term whose associated monomial is $m$ (we will assume that like terms have been collected, so this is unique).  As a short-hand, we will say that a monomial $m$ is ``in $q$" if the associated coefficient of $m$ in $q$ is non-zero.

The \emph{zero polynomial} is defined as a polynomial which has no monomials in it.  While the zero polynomial clearly computes the zero function, non-zero polynomials can sometimes also compute the zero function over the domain of $y$, and thus these are related but distinct concepts \footnote{For example, $y_1(1-y_1)$ computes the zero function when the domain of $y_1$ is $\{0,1\}$ but is clearly not the zero polynomial.}. 

A polynomial is called \emph{non-negative} if the coefficients of each of its monomials are non-negative.  Non-negativity is related to the monotonicity of arithmetic circuits in the following way:
\begin{fact}
\label{fact:pos_coeff}
If $\Phi$ is a monotone arithmetic circuit over $y$ (such as an SPN with $y = f$), then $q_\Phi$ is a non-negative polynomial.
\end{fact}


We will define the \emph{scope} of a polynomial $q$ in $y$, denoted by $y_q$, to be the set of variables which appear as factors in at least one of its monomials.  Note that for an node $u$ in an arithmetic circuit, the scope $y_u$ of $u$ can easily be shown to be a superset of the scope of its output polynomial $q_u$ (i.e. $y_{q_u}$), but it will not be equal to $y_{q_u}$ in general.


A central concept in our analysis of D\&C SPNs will be that of multilinearity, which is closely related to the decomposability condition.

\begin{definition}(Multilinear Polynomial)
A polynomial $q$ in $y$ is \emph{multilinear} if the degree of each element of $y$ is at most one in each monomial in $q$.
\end{definition}
For example, $y_1 + y_2 y_3$ is a multilinear polynomial.  

Some more interesting examples of multilinear polynomials include the permanent and determinant of a matrix (where we view the entries of the matrix as the variables).  

\begin{definition}(Multilinear Arithmetic Circuit)
If every node of an arithmetic circuit $\Phi$ over $y$ computes a multilinear polynomial in $y$, $\Phi$ is said to be a \emph{(semantically) multilinear arithmetic circuit}.  And if for every product node in $\Phi$, the scopes of its child nodes are pair-wise disjoint, $\Phi$ is said to be a \emph{syntactically multilinear arithmetic circuit}.
\end{definition}

It is easy to show that a syntactically multilinear arithmetic circuit is also a semantically multilinear circuit. However, it is an open question as to whether one can convert a semantically multilinear arithmetic circuit into a syntactically multilinear one without increasing its size by a super-polynomial factor \cite{raz2008lower}. In the case of formulas however, given a semantically multilinear formula of size $s$ one can transform it into an equivalent syntactically multilinear formula of size at most $s$ \cite{raz2004multi}. 


It should be obvious by this point that there is an important connection between syntactic multilinearity and decomposability.  In particular, if our univariate functions of the $x_i$'s are all identity functions, then scope and dependency-scope become equivalent, and thus so do syntactic multilinearity and decomposability.

Given this observation we have that a monotone syntactically multilinear arithmetic circuit over $x$ can be viewed as a decomposable SPN.

A somewhat less obvious fact (which will be very useful later) is that any decomposable SPN over 0/1-valued $x_i$'s can be viewed as a syntactically multilinear arithmetic circuit over $x$, of a similar size and depth.  To see this, note that any arbitrary univariate function $g$ of a 0/1-valued variable $z$ can always be written as an affine function of $z$, i.e. of the form $az + b$ with $a = g(1) - g(0)$ and $b = g(0)$.  Thus we can replace each node computing a univariate function of some $x_i$ with a subcircuit computing this affine function, and this yields a (non-monotone) syntactically multilinear arithmetic circuit over $x$, with a single additional layer of sum nodes (of size $\bigO(n)$).

An extension of the concept of multilinearity is that of set-multilinearity \citep[e.g.][]{shpilka2010arithmetic}.  To define set-multilinearity we must first define some additional notation which we will carry through the rest of the paper.

Let $G_1, G_2, G_3,..., G_k$ be a partitioning of the elements of $y$ into disjoint sets.  The \emph{set-scope} $G_q$ of a polynomial $q$ is the sub-collection of the collection $\{G_1, ..., G_k\}$ defined as consisting of those sets $G_i$ which have some element in common with the scope $y_q$ of $q$.  i.e. $G_q = \{ G_i : y_q \cap G_i \neq \emptyset\}$. Similarly, the \emph{set-scope} $G_u$ of a node $u$ in an arithmetic circuit $\Phi$ is the sub-collection of the collection $\{G_1, ..., G_k\}$ defined as consisting of those sets $G_i$ which have some element in common with the scope of $u$.  i.e. $G_u = \{ G_i : y_u \cap G_i \neq \emptyset\}$.

\begin{definition}(Set-multilinear polynomial)
A polynomial is set-multilinear if each of its monomials has exactly one factor from each of the $G_i$'s in its set-scope.
\end{definition}

For example, $3 y_1 y_3 - y_2 y_4$ is a set-multilinear polynomial when $G_1 = \{y_1, y_2\}, G_2 = \{y_3, y_4\}$, while $y_1 y_2 + 2 y_2 y_4$ is not.  The permanent and determinant of a matrix also turn out to be non-trivial examples of set-multilinear polynomials, if we define the collection of sets so that $G_i$ consists of the entries in the $i^{th}$ row of the matrix.

\begin{definition}(Set-multilinear arithmetic circuits)
An arithmetic circuit is called \emph{(semantically) set-multilinear} if each of its nodes computes a set-multilinear polynomial. An arithmetic circuit $\Phi$ is called \emph{syntactically set-multilinear} if it satisfies the following two properties:
\begin{itemize}
\item for each product node $u$ in $\Phi$, the set-scopes of the children of $u$ are pairwise disjoint
\item for each sum node $u$, the set-scopes of the children of $u$ are all the same.
\end{itemize}
\end{definition}

A crucial observation is that the concepts of set-multilineary in arithmetic circuits and decomposability and completeness in SPNs are even more closely related than syntactic multilinearity is to decomposability.  In particular, it is not hard to see that if we take $k = n$, $y = f$ and $G_i = f_i$, then set-scope (of nodes) and dependency-scope become analogous concepts, 
and D\&C SPNs correspond precisely to monotone syntactically set-multilinear arithmetic circuits in $f$.  Because of the usefulness of this connection, we will use the above identifcations for the remainder of this paper whenever we discuss set-multilinearity in the specific context of SPNs.

This connection also motivates a natural definition for the dependency-scope for polynomials over $f$.  In particular, the dependency-scope of the polynomial $q$ over $f$ will be defined as the set of variables on which the members of $q$'s scope $f_u$ depend.  We will denote the dependency-scope by $x_q$.

\section{Analysis of Validity, Decomposability and Completeness}
\label{sec:analysis}

In this section we give a novel analysis of the relationship between validity, decomposability and completeness, making use of many of the concepts from circuit theory reviewed in the previous section.  

First, we will give a quick result which shows that an incomplete SPN can always be efficiently transformed into a complete one which computes the same function of $x$.  Note that this is not a paradoxical result, as the new SPN will behave differently than the original one when used to evaluate integrals (in the sense of definition of valid SPNs in Section \ref{sec:definitions}).
\begin{proposition}
\label{prop:completeness_transform}
Given an SPN $\Phi$ of size $s$ there exists a complete SPN $\Phi'$ of size $s + n + k \in \bigO(s^2)$, and an expanded set/tuple of univariate functions $f'$ s.t. $q_{\Phi'}(f'(x)) = q_\Phi(f(x))$ for all values of $x$, where $k$ is the sum over the fan-in's of the sum nodes of $\Phi$.  Moreover, $\Phi'$ is decomposable if $\Phi$ is.
\end{proposition}
So in some sense we can always get completeness ``for free", and of the two properties, decomposability will be the one which actually constrains SPNs in a way that affects their expressive efficiency.

Unlike with decomposability and completeness, validity depends on the particular definitions of univariate functions making up $f$, and thus cannot be described in purely structural terms like set-multilinearity. 
This leads us to propose a slightly stronger condition which we call \emph{strong validity}, which is independent of the particular choice of univariate functions making up $f$.

\begin{definition}
An SPN $\Phi$ is said to be \emph{strongly valid} if it is valid for every possible choice of the univiariate functions making up $f$.  Note: only the values computed by each $f_{i,j}$ are allowed to vary here, not the identities of the dependent variables.
\end{definition}




The following theorem establishes the fundamental connection between set-multilinearity and strong validity.
\begin{theorem}
\label{thm:strongvalidity_outputpoly}
Suppose the elements of $x$ are all non-trivial variables (as defined below). Then an SPN $\Phi$ is strongly valid if and only if its output polynomial is set-multilinear.
\end{theorem}

A variable $x_i$ is \emph{non-trivial} if there are at least two disjoint subsets of $x_i$'s range $R_i$ which have finite and non-zero measure under $M_i$.

Non-triviality is a very mild condition.  In the discrete case, it is equivalent to requiring that there is more than one element of the range set $R_i$ which has mass under the associated measure.  Trivial variables can essentially be thought of as ``constants in disguise", and we can easily just replace them with constant nodes without affecting the input-output behavior of the circuit.

It is worth noting that the non-triviality hypothesis is a necessary one for the forward direction of Theorem \ref{thm:strongvalidity_outputpoly} (although not the reverse direction).  To see this, consider for example the SPN $\Phi$ which computes $f_{1,1}(x_1)^2 f_{2,1}(x_2)$ in the obvious way, where $R_1 = \{1\}$ and $R_2 = \{0,1\}$, and the $M_i$ are the standard counting measures.  While $\Phi$'s output polynomial $q_\Phi$ is not set-multilinear by inspection, it is relatively easy to show that $\Phi$ is indeed strongly valid, as it is basically equivalent to $c f_{2,1}(x_2)$ for a constant $c$.

While Theorem \ref{thm:strongvalidity_outputpoly} is interesting by itself as it provides a complete characterization of strong validity in terms of purely algebraic properties of an SPN's output polynomial, its main application in this paper will be to help prove the equivalence of strong validity with the decomposability and completeness conditions.

Note that such an equivalence does not hold for \emph{standard} validity, as was first demonstrated by \cite{poon2011sum}.  To see this, consider the SPN which computes the expression $(f_{1,1}(x_1)f_{1,2}(x_1) + 1)f_{2,1}(x_2)$ in the obvious way, where the $x_i$ are 0/1-valued, the $M_i$ are the standard counting measures, and $f_{1,1}(x_1) = x_1$, $f_{1,2}(x_1) = \overline{x_1} = 1-x_1$, and $f_{2,1}(x_2) = x_2$.  Clearly this SPN is neither decomposable nor complete, and yet an exhaustive case analysis shows that it is valid for these particular choices of the $f_{i,j}$'s.

Before we can prove the equivalence of strong validity with the decomposability and completeness conditions, we need to introduce another mild hypothesis which we call ``non-degeneracy".

\begin{definition}
A monotone arithmetic circuit (such as an SPN) is called \emph{non-degenerate} if all of its weights and constants are non-zero (i.e. strictly positive).
\end{definition}

Like non-triviality, non-degeneracy is a very mild condition to impose, since weights which are zero don't actually ``affect" the output.  Moreover, there is a simple and size-preserving procedure which can transform a degenerate monotone arithmetic circuit to a non-degenerate one which computes precisely the same output polynomial, and also preserves structural properties like decomposability and completeness in SPNs.  The procedure is as follows.  First we remove all edges with weight $0$.  Then we repeatedly remove any nodes with fan-out 0 (except for the original output node) or fan-in 0 (except input node and constant nodes), making sure to remove any product node which is a parent of a node we remove.  It is not hard to see that deletion of a node by this procedure is a proof that it computes the zero-polynomial and thus doesn't affect the final output.  

Without non-degeneracy, the equivalence between strong validity and the decomposability and completeness conditions does not hold for SPNs, as can be seen by considering the SPN $\Phi$ which computes the expression $0f_{1,1}(x_1)^2 + f_{1,1}(x_1) f_{2,1}(x_2)$ in the obvious way, where the $x_i$ are 0/1-valued and the $M_i$ are the standard counting measures, and the $f_{i,j}$'s are the identity function (i.e. $f_{i,j}(x_i) = x_i$).  Because the output polynomial $q_\Phi$ of $\Phi$ is equivalent to $f_{1,1}(x_1) f_{2,1}(x_2)$, it is indeed valid.  However, the product node within $\Phi$ which computes $f_{1,1}(x_1)^2$ violates the decomposability condition, even though this computation is never actually ``used" in the final output (due to how it is weighted by 0).

Non-degeneracy allows us to prove many convenient properties, which are given in the lemma below.

\begin{lemma}
\label{lemma:non-degen}
Suppose $\Phi$ is a non-degenerate monotone arithmetic circuit.  Denote by $r$ the root of $\Phi$, and $\{u_i\}_i$ its child nodes.

We have the following facts:
\begin{enumerate}
\item Each of the $\Phi_{u_i}$'s are non-degenerate monotone arithmetic circuits.

\item If $r$ is a product node, the set of monomials in $q_r$ is equal to the set consisting of every possible product formed by taking one monomial from each of the $q_{u_i}$'s.   NOTE: This is true even for degenerate circuits.

\item If $r$ is a sum node, the set of monomials in $q_r$ is equal to the union over the sets of monomials in the $q_{u_i}$'s.

\item $q_r$ is not the zero polynomial.

\item The set-scope of $r$ is equal to the set-scope of $q_r$.
\end{enumerate}

\end{lemma}

We are now in a position to prove the following theorem. 

\begin{theorem}
\label{theorem:sml_output_equiv_setmulti}
A non-degenerate monotone arithmetic circuit $\Phi$ has a set-multilinear output polynomial if and only if it is syntactically set-multilinear. 
\end{theorem}

Given this theorem, and utilizing the previously discussed connection between syntactic set-multilinearity and the decomposability and completeness conditions, the following corollary is immediate:
\begin{corollary}
\label{lemma:sml_output_equiv_candc}
A non-degenerate SPN $\Phi$ has a set-multilinear output polynomial if and only if it is decomposable and complete. 
\end{corollary}

And from this and Theorem \ref{thm:strongvalidity_outputpoly}, we have a 3-way equivalence between strong validity, the decomposability and completeness conditions, and the set-multilinearity of the output polynomial.  This is stated as the following theorem.

\begin{theorem}
\label{thm:3way_equiv}
Suppose $\Phi$ is a non-degenerate SPN whose input variables (the elements of $x$) are all non-trivial.  Then the following 3 conditions are equivalent:
\begin{enumerate}
\item $\Phi$ is strongly valid
\item $\Phi$ is decomposable and complete
\item $\Phi$'s output polynomial is set-multilinear
\end{enumerate}
\end{theorem}

Because SPNs can always be efficiently transformed so that the non-degeneracy and non-triviality hypotheses are both satisfied (as discussed above), this equivalence between strong validity and the D\&C conditions makes the former easy to verify (since the D\&C conditions themselves are). 

However, as we will see in later sections, decomposability and completeness are restrictive conditions that limit the expressive power of SPNs in a fundamental way. And so a worthwhile question to ask is whether a set of efficiently testable criteria exist for verifying standard/weak validity.


We will shed some light on this question by proving a result which shows that a criterion cannot be both efficiently testable and capture all valid SPNs, provided that $P \neq NP$.  A caveat to this result is that we can only prove it for a slightly extended definition of SPNs where negative weights and constants are permitted.

\begin{theorem}
\label{thm:validity_coNP-hard}
Define an \emph{extended SPN} as one which is allowed to have negative weights and constants.  The problem of deciding whether a given extended SPN is valid is co-NP-hard.
\end{theorem}


We leave it as an open question as to whether a similar co-NP-hardness property holds for validity checking of standard SPNs.

\section{Focusing on D\&C SPNs}

One of the main goals of this paper is to advance the understanding of the expressive efficiency of SPNs.  In this section we explore possible directions we can take towards this goal, and ultimately propose to focus exclusively on D\&C SPNs.

It is well known that standard arithmetic circuits can efficiently simulate Boolean logic circuits with only a constant factor overhead.  Thus they are as efficient at computing a given function as any standard model of computation, up to a polynomial factor.  However, we cannot easily exploit this fact to study SPNs, as this simulation requires negative weights, and the weights of an SPN are constrained to be non-negative (i.e. they are monotone arithmetic circuits).  And while SPNs have access to non-negative valued univariate functions of the input which standard monotone arithmetic circuits do not, this fact cannot obviously be used to construct a simulation of Boolean logic circuits.


Another possible way to gain insight into general SPNs would be to apply existing results for monotone arithmetic circuits.  However, a direct application of such results is impossible, as SPNs are monotone arithmetic circuits over $f$ and not $x$, and indeed their univariate functions can compute various non-negative functions of $x$ (such as $1-x_i$ for values of $x_i$ in $\{0,1\}$) which a monotone circuit could not.

But while it seems that the existing circuit theory literature doesn't offer much insight into general SPNs, there are many interesting results available for multilinear and set-multilinear arithmetic circuits.  And as we saw in Section \ref{sec:analysis}, these are closely related to D\&C SPNs.  

Moreover, it makes sense to study D\&C SPNs, as they are arguably the most interesting class of SPNs, both from a theoretical and practical perspective. Indeed, the main reason why SPNs are interesting and useful in the first place is that \emph{valid SPNs} avoid the intractability problems that plague conventional deep models like Deep Boltzmann Machines.  Meanwhile the D\&C conditions are the only efficiently testable conditions for ensuring validity that we are aware of, and as we showed in Section \ref{sec:analysis}, they are also \emph{necessary} conditions for a slightly strengthened notion of validity.

Thus, D\&C SPNs will be our focus for the rest of the paper.

\section{Capabilities of D\&C SPNs}
\label{sec:capabilities}

Intuitively, D\&C SPNs seem very limited compared to general arithmetic circuits.  In addition to being restricted to use non-negative weights and constants like general SPNs, decomposability heavily restricts the kinds of structure the networks can have, and hence the kinds of computations they can perform.  For example, something as simple as squaring the number computed by some node $u$ becomes impossible.

In order to address the theoretical question of what kinds of functions D\&C SPNs \emph{can} compute efficiently, despite their apparent limitations, we will construct explicit D\&C SPNs that efficiently compute various example functions.  

This is difficult to do directly because the decomposability condition prevents us from using the basic computational operations we are accustomed to working with when designing algorithms or writing down formulae. To overcome this difficulty we will provide a couple of related examples of computational systems which we will show can be efficiently simulated by SPNs.  These systems will be closer to more traditional models of computation like state-space machines, so that our existing intuitions about algorithm design will be more directly applicable to them.

The first such system we will call a Fixed-Permutation Linear Model (FPLM), which works as follows.  We start by initializing a ``working vector" $v$ with a value $a$, and then we process the input (the $x_i$'s) in sequence, according to a fixed order given by a permutation $\pi$ of $[n]$.  At each stage we multiply $v$ by a matrix which is determined by the value of the current $x_i$.  After seeing the whole input, we then take the inner product of $v$ with another vector $b$, which gives us our real-valued output.

More formally, we can define FPLMs as follows.  
\begin{definition} 
A Fixed-Permutation Linear Model (FPLM) will by defined by a fixed permutation $\pi$ of $[n]$, a `dimension' constant $k$ (which in some sense measures the size of the FPLM), vectors $a,b \in \Real_{\geq 0}^k$ and for each $i \in [n]$, a matrix-valued function $T_i$ from $x_i$ to $\Real_{\geq 0}^{k \times k}$.   The output of a FPLM is defined as $b^\top T_{\pi(n)}(x_{\pi(n)}) T_{\pi(n-1)}(x_{\pi(n-1)}) \cdots T_{\pi(1)}(x_{\pi(1)}) a$.
\end{definition}

An FPLM can be viewed as a computational system which must process its input in a fixed order and maintains its memory/state as a $k$-dimensional vector.  Crucially, an FPLM cannot revisit inputs that it has already processed, which is a similar limitation to the one faced by read-once Turing Machines.  The state vector can be transformed at each stage by a linear transformation which is a function of the current input.  While its $k$-dimensional state vector allows an FPLM to use powerful distributed representations which clearly possess enough information capacity to memorize the input seen so far, the fundamental limitation of FPLMs lies in their limited tools for manipulating this representation.  In particular, they can only use linear transformations (given by matrices with positive entries).  If they had access to arbitrary transformations of their state then it is not hard to see that \emph{any} function could be efficiently computed by them.

The following result establishes that D\&C SPNs can efficiently simulate FPLMs.
\begin{proposition}
\label{prop:FPLM_simulation}
Given a FPLM of dimension $k$ there exists a D\&C SPN of size $\bigO(nk^2)$ which computes the same function.
\end{proposition}

Thus D\&C SPNs are at least as expressively efficient as FPLMs.  This suggests the following question: are they strictly more expressively efficient than FPLMs, or are they equivalent?  It turns out that they are more expressively efficient.  We sketch a proof of this fact below.

Suppose that $x$ takes values in $\{0,1\}^n$.  As observed in Section \ref{sec:polys_and_multi}, this allows us to assume without loss of generality that any univariate function of one of the $x_i$'s is affine in $x_i$.  In particular, we can assume that the matrix-valued functions $T_i$ used in FPLMs are affine functions of the respective $x_i$'s.  In this setting it turns out that FPLMs can be viewed as a special case of a computational system called ``ordered syntactically multilinear branching programs", as they are defined by \citet{Jansen08}.  \citet{Jansen08} showed that there exists a polynomial function in $x$ whose computation by such a system requires exponential size (corresponding to an FPLM with an exponentially large dimension $k$).  Moreover, this function is computable by a polynomially sized monotone syntactically multilinear arithmetic.  As observed in Section \ref{sec:polys_and_multi}, such a circuit can be viewed as a decomposable SPN whose univariate functions are just identity functions.  Then using Proposition \ref{prop:completeness_transform} we can convert such a decomposable SPN to a D\&C SPN while only squaring its size.  So the polynomial provided by \citet{Jansen08} is indeed computed by a D\&C SPN of polynomial size, while requiring exponential size to be computed by a FPLM, thus proving that D\&C SPNs are indeed more expressively efficient.

Given this result, we see that FPLMs do not fully characterize the capabilities of D\&C SPNs.  Nevertheless, if we can construct an FPLM which computes some function efficiently, this constitutes proof of existence of a similarly efficient D\&C SPN for computing said function.


To make the construction of such FPLMs simpler, we will define a third computational system which we call a Fixed-Permutation State-Space Model (FPSSM) which is even easier to understand than FPLMs, and then show that FPLMs (and hence also D\&C SPNs) can efficiently simulate FPSSMs.

An FPSSM works as follows.  We initialize our ``working state" $u$ as $c$, and then we process the input (the $x_i$'s) in sequence, according to a fixed order given by the permutation $\pi$ of $[n]$.  At each stage we transform $u$ by computing $g_{\pi(i)}( x_{\pi(i)}, u )$, where the transition function $g_{\pi(i)}$ can be defined arbitrarily.  After seeing the whole input, we then decode the state $u$ as the non-negative real number $h(u)$.

More formally we have the following definition.

\begin{definition} 
A Fixed-Permutation State-Space Model (FPSSM) will by defined by a fixed permutation $\pi$ of $[n]$, a `state-size' constant $k$ (which in some sense measures the size of the FPSSM), an initial state $c \in [k]$, a decoding function $h$ from $[k]$ to $\Real_{\geq 0}$, and for each $i \in [n]$ an arbitrary function $g_i$ which maps values of $x_i$ and elements of $[k]$ to elements of $[k]$.  The output of an FPSSM will be defined as $h( g_{\pi(n)}( x_{\pi(n)}, g_{\pi(n-1)}( \cdots g_{\pi(1)}( x_{\pi(1)}, c ) \cdots )))$ for an arbitrary function $h$ mapping elements of $[k]$ to $\Real_{\geq 0}$.
\end{definition}

FPSSMs can be seen as general state-space machines (of state size $k$), which like FPLMs, are subject to the restriction that they must process their inputs in a fixed order which is determined ahead of time, and are not allowed to revisit past inputs.  If the state-space is large enough to be able to memorize every input seen so far, it is clear that FPSSMs can compute any function, given that their state-transition function can be arbitrary.  But this would require their state-size constant $k$ to grow exponentially in $n$, as one needs a state size of $2^b$ in order to memorize $b$ input bits.  FPSSMs of realistic sizes can only memorize a number of bits which is logarithmic in $n$.  And this, combined with their inability to revisit past inputs, clearly limits their ability to compute certain functions efficiently.  This is to be contrasted with FPLMs, whose combinatorial/distributed state have a high information capacity even for small FPLMs, but are limited instead in how they can manipulate this state.

The following result establishes that FPLMs can efficiently simulate FPSSMs.
\begin{proposition}
\label{prop:FPSSM_simulation}
Given a FPSSM of state-size $k$ there exists a FPLM of dimension $k$ which computes the same function.
\end{proposition}

Note that this result also implies that FPSSMs are no more expressively efficient than FPLMs, and are thus strictly less expressively efficient than D\&C SPNs.

The following Corollary follows directly from Propositions \ref{prop:FPLM_simulation} and \ref{prop:FPSSM_simulation}:
\begin{corollary}
\label{cor:FPSSM_sim_corr}
Given a FPSSM of state-size $k$ there exists a D\&C SPN of size $\bigO(nk^2)$ which computes the same function.
\end{corollary}

Unlike with D\&C SPNs, our intuitions about algorithm design readily apply to FPSSMs, making it easy to directly construct FPSSMs which implement algorithmic solutions to particular problems.  For example, suppose we wish to compute the number of inputs whose value is equal to 1. We can solve this with an FPSSM with a state size of $k = n$ by taking $\pi$ to be the identity permutation, and the state $u$ to be the number of 1's seen so far, which we increment whenever the current $x_i$ has value 1.  We can similarly compute the parity of the number of ones (which is a well known and theoretically important function often referred to simply as ``PARITY") by storing the current number of them modulo 2, which only requires the state size $k$ to be $1$.   We can also decide if the majority of the $x_i$'s are 1 (which is a well known function and theoretically important often referred to as ``MAJORITY" or ``MAJ") by storing a count of the number of ones (which requires a state size of $k = n$), and then outputting $1$ if $s \geq n/2$ and $0$ otherwise.

It is noteworthy that the simulations of various models given in this section each require an D\&C SPN of depth $n$.  However, as we will see in Section \ref{sec:depth_reduction}, the depth of any D\&C SPN can be reduced to $\bigO(log(n)^2)$, while only increasing its size polynomially.

\section{Separating Depth 3 From Higher Depths}
\label{sec:depth3separation}

The only prior work on the expressive efficiency of SPNs which we are aware of is that of \citet{DelalleauBengio2011}.  In that work, the authors give a pair of results which demonstrate a difference in expressive efficiency between D\&C SPNs of depth 3, and those of higher depths.

Their first result establishes the existence of an $n$-dimensional function $g$ (for each $n$) which can be computed by a D\&C SPN of size $\bigO(n)$ and depth $\bigO(\log(n))$, but which requires size $\Omega(2^{\sqrt{n}})$ to be computed by a D\&C SPN of depth\footnote{Note that in our presentation the input layer counts as the first layer and contributes to the total depth.} 3.

In their second result they show that for each $d \geq 4$ there is an $n$-dimensional function $h_d$ which can be computed by a D\&C SPN of size $\bigO(dn)$ and depth $d$, but which requires size $\Omega(n^d)$ to be computed by a D\&C SPN of depth 3.

It is important to note that these results do not establish a separation in expressive efficiency between D\&C SPNs of any two depths both larger than 3 (e.g. between depths 4 and 5).  So in particular, despite how the size lower bound increases with $d$ in their second result\footnote{As shown by our Theorem \ref{thm:SPN_depth_hierarchy}, there is a much stronger separation between depths 3 and 4 than is proved by \citet{DelalleauBengio2011} to exist between depths $3$ and $d$ for \emph{any} $d \geq 4$, and thus this apparent increase in the size of their lower bound isn't due to the increasing power of D\&C SPNs with depth so much as it an artifact of their particular proof techniques.} this \emph{does not} imply that the set of efficiently computable functions is larger for D\&C SPNs of depth $d + k$ than for those of depth $d$, for any $k > 0$, except when $d \leq 3$.  This is to be contrasted with our much stronger ``depth hierarchy" result (Theorem \ref{thm:SPN_depth_hierarchy} of Section \ref{sec:depth_hierarchy}) which shows that D\&C SPNs do in fact have this property (even with $k = 1$) \emph{for all} choices of $d$, where depth is measured in terms of product-depth.




In the next subsection we will show how basic circuit theoretic techniques can be used to give a short proof of a result which is stronger than both of the separation results of \citet{DelalleauBengio2011}, using example functions which are natural and simple to understand.  Beyond providing a simplified proof of existing results, this will also serve as a demonstration of some of the techniques underlying the more advanced results from circuit theory which we will later make use of in Sections \ref{sec:depth_hierarchy} and \ref{sec:formula_limitation}.

Moreover, by employing these more general and powerful proof techniques, we are able to prove a stronger result which seperates functions that can be efficiently \emph{approximated} by D\&C SPNs of depth 3 from those which can be computed by D\&C SPNs of depth 4 and higher.  This addresses the open question posed by \citet{DelalleauBengio2011}.

\subsection{Basic separation results}

We begin by defining some basic concepts and notation which are standard in circuit theory.

For an arbitrary function $g$ of $x$, and a partition $(A,B)$ of the set $[n]$ of indices of the elements of $x$, define $M^{A,B}_g$ to be the $2^{|A|}$ by $2^{|B|}$ matrix of values that $g$ takes for different values of $x$, where the rows of $M^{A,B}_g$ are indexed by possible values of $x_A$, and the columns of $M^{A,B}_g$ are indexed by possible values of $x_B$.

$M^{A,B}_g$ is called a ``communication matrix" in the context of communication complexity, and appears frequently as a tool to prove lower bounds.  Its usefulness in lower bounding the size of D\&C SPNs of depth $3$ is established by the following theorem.

\begin{theorem}
\label{thm:3layer_rank_bound}
Suppose $\Phi$ is a D\&C SPN of depth 3 with $k$ nodes in its second layer.  For any partition $(A, B)$ of $[n]$ we have $k \geq \rank \left( M^{A,B}_{q_{\Phi}(f(x))} \right)$.
\end{theorem}

Note that the proof of this theorem doesn't use the non-negativity of the weights of the SPN, and thus applies to the ``extended" version of SPNs discussed in Section \ref{sec:analysis}.

We will now define the separating function which we will use to separate the expressive efficiency of depth 3 and 4 D\&C SPNs.  

Define $H_1 = \{1,2,...,n/2\}$ and $H_2 = \{n/2+1, n/2+2, ..., n\}$.  We will define the function $\EQUAL$ for $\{0,1\}^n$-valued $x$ to be $1$ when $x_{H_1} = x_{H_2}$ (i.e. the first half of the input is equal to the second half) and $0$ otherwise.  

Observe that $M^{H_1,H_2}_{\EQUAL} = I$ and so this matrix has rank $2^{n/2}$.  This gives the following simple corollary of the above theorem:
\begin{corollary}
\label{cor:lowerbound_exact_equal}
Any D\&C SPN of depth 3 with computes $\EQUAL(x)$ must have at least $2^{n/2}$ nodes in its second layer.
\end{corollary}

Meanwhile, $\EQUAL(x)$ is easily shown to be \emph{efficiently} computed by a D\&C SPN of depth 4.  This is stated as the following proposition.

\begin{proposition}
\label{prop:equal_upper_bound}
$\EQUAL$ can be computed by an D\&C SPN of size $O(n)$ and depth 4.
\end{proposition}

Note that the combination of Corollary \ref{cor:lowerbound_exact_equal} and Proposition \ref{prop:equal_upper_bound} gives a stronger separation result than both of the aforementioned results of \citet{DelalleauBengio2011}.  Our result also has the advantage of using an example function which is easy to interpret, and can be easily extended to prove separation results for other functions which have a high rank communication matrix.

\subsection{Separations for approximate computation}

An open question posed by \citet{DelalleauBengio2011} asked whether a separation in expressive efficiency exists between D\&C SPNs of depth 3 and 4 if the former are only required to compute an approximation to the desired function.  In this section we answer this question in the affirmative by making use of Theorem \ref{thm:3layer_rank_bound} and an additional technical result which lower bounds the rank of the perturbed versions of the identity matrix.

\begin{theorem}
\label{thm:rank_approx_equal}
Suppose $\Phi$ is a D\&C SPN of depth 3 whose associated distribution is such that each value of $x$ with $\EQUAL(x) = 1$ has an associated probability between $a/2$ and $a$ for some $a > 0$ (so that all such values of $x$ have roughly equal probability), and that the total probability $\delta$ of all of the values of $x$ satisfying $\EQUAL(x) = 0$ obeys $\delta \leq 1/4$ (so that the probability of drawing a sample with $\EQUAL(x) = 0$ is $\leq 1/4$).  Then $\Phi$ must have at least $2^{n/2 - 2}/3$ nodes in its second layer.
\end{theorem}

To prove this result using Theorem \ref{thm:3layer_rank_bound} we will make use of the following lemma which lower bounds the rank of matrices of the form $I + D = M^{H_1,H_2}_{\EQUAL} + D$ for some ``perturbation matrix" $D$, in terms of a measure of the total size of the entries of $D$.

\begin{lemma}
\label{lemma:rank_perturb_lemma}
Suppose $D \in \Real^{k \times k}$ is a real-valued matrix such that $\sum_{i,j} |[D]_{i,j}| = \Delta$ for some $\Delta \geq 0$.  Then $\rank(I + D) \geq k/2 - \Delta/2$.
\end{lemma}

\section{Depth Analysis}
\label{sec:depth_analysis}

\subsection{A depth hierarchy for D\&C SPNs}
\label{sec:depth_hierarchy}

In this section we show that D\&C SPNs become more expressively efficient as their product-depth\footnote{Product-depth is defined in Section \ref{sec:arithmetic_circuit_defs}.  Note that it can be shown to be equivalent to standard depth up to a factor of 2 (e.g. by `merging' sum nodes that are connected as parent/child).} $d$ increases, in the sense that the set of efficiently computable density functions expands as $d$ grows.  This is stated formally as follows:

\begin{theorem}
\label{thm:SPN_depth_hierarchy}
For every integer $d \geq 1$ and input size $n$ there exists a real-valued function $g_{d+1}$ of $x$ such that:
\begin{enumerate}
\item There is a D\&C SPN of product-depth $d+1$ and size $\bigO(n^2)$ which computes $g_{d+1}$ for all values of $x$ in $\{0,1\}^n$, where the SPN's univariate functions $f$ consist only of identity functions.
\item For any choice of the univariate functions $f$, a D\&C SPN of product-depth $d$ that computes $g_{d+1}$ for all values of $x$ in $\{0,1\}^n$ must be of size $n^{\Omega(\log(n)^{1/2d})}$ (which is super-polynomial in $n$).
\end{enumerate}
\end{theorem}


Previously, the only available result on the relationship of depth and expressive efficiency of D\&C SPNs has been that of \citet{DelalleauBengio2011}, who showed that D\&C SPNs of depth 3 are less expressively efficient than D\&C SPNs of depth 4. 

An anologous result seperating very shallow networks from deeper ones also exists for neural networks.   In particular, it is known that under various realistic constraints on their weights, threshold-based neural networks with one hidden layer (not counting the output layer) are less expressively efficient those with 2 or more hidden layers \citep{HMPST, Forster02}.  

More recently, \citet{RBM_Expressive_Martens} showed that Restricted Boltzmann Machines are incapable of efficiently capturing certain simple distributions, which by the results of \citet{Deep_Expressive_Martens}, can be efficiently captured by Deep Boltzmann Machines.

A ``depth-hierarchy" property analogous to Theorem \ref{thm:SPN_depth_hierarchy} is believed to hold for various other deep models like neural networks, Deep Boltzmann Machines \citep{DBMs}, and Sigmoid Belief Networks \citep{radford_SBN}, but has never been proven to hold for any of them.  Thus, to the best of our knowledge, Theorem \ref{thm:SPN_depth_hierarchy} represents the first time that a practical and non-trivial deep model has been rigorously shown to gain expressive efficiency with each additional layer of depth added. 

To prove Theorem \ref{thm:SPN_depth_hierarchy}, we will make use of the following analogous result which is a slight modification of one proved by \citet{raz2009lower} in the context of multilinear circuits. 

\begin{theorem}(Adapted from Theorem 1.2 of \citet{raz2009lower})
\label{thm:raz_depth_hierarchy}
For every integer $d \geq 1$ and input size $n$ there exists a real-valued function $g_{d+1}$ of $x$ such that:
\begin{enumerate}
\item There is a monotone syntactically multilinear arithmetic circuit over $x$ of product-depth $d+1$, size $\bigO(n)$ which computes $g_{d+1}$ for all values of $x$ in $\Real^n$.
\item Any syntactically multilinear arithmetic circuit over $x$ of product-depth $d$ that computes $g_{d+1}$ for all values of $x$ in $\Real^n$ must be of size $n^{\Omega(\log(n)^{1/2d})}$.
\end{enumerate}
\end{theorem}

Note that the original Theorem 1.2 from \citet{raz2009lower} uses a slightly different definition of arithmetic circuits from ours (they do not permit weighted connections), and the constructed circuits are not stated to be monotone.  However we have confirmed with the authors that their result still holds even with our definition, and the circuits constructed in their proof are indeed monotone \citep{raz_personal}.  

There are several issues which must be overcome before we can use Theorem \ref{thm:raz_depth_hierarchy} to prove Theorem \ref{thm:SPN_depth_hierarchy}.  The most serious one is that syntactically multilinear arithmetic circuits are not equivalent to D\&C SPNs as either type of circuit has capabilities that the other does not.  Thus the ability or inability of syntactically multilinear arithmetic circuits to compute certain functions does not immediately imply the same thing for D\&C SPNs.  

To address this issue, we will consider the case where $x$ is binary-valued (i.e. takes values in $\{0,1\}^n$) so that we may exploit the close relationship which exists between syntactically multilinear arithmetic circuits and decomposable SPNs over binary-valued inputs $x$ (as discussed in Section \ref{sec:polys_and_multi}).  

Another issue is that Theorem \ref{thm:raz_depth_hierarchy} deals only with the hardness of computing certain functions over all of $\Real^n$ instead of just $\{0,1\}^n$ (which could be easier in principle).  However, it turns out that for circuits with multilinear output polynomials, computing a function over $\{0,1\}^n$ is equivalent to computing it over $\Real^n$, as is established by the following lemma.

\begin{lemma}
\label{lemma:mult_ident_test}
If $q_1$ and $q_2$ are two multilinear polynomials over $y = (y_1,...,y_\ell)$ with $q_1(y) = q_2(y) \quad \forall y \in \{0,1\}^\ell$, then $q_1(y) = q_2(y) \quad \forall y \in \Real^\ell$.
\end{lemma}

With these observations in place the proof of Theorem \ref{thm:SPN_depth_hierarchy} from Theorem \ref{thm:raz_depth_hierarchy} becomes straightforward (and is given in the appendix).

\subsection{The limits of depth}
\label{sec:depth_reduction}

Next, we give a result which shows that the depth of any polynomially sized D\&C SPN can be essentially compressed down to  $\bigO( \log(n)^2 )$, at the cost of only a \emph{polynomial} increase in its total size.  Thus, beyond this sublinear threshold, adding depth to a D\&C SPN does not increase the set of functions which can be computed efficiently (where we use this term liberally to mean ``with polynomial size").  Note that this does not contradict Theorem \ref{thm:SPN_depth_hierarchy} from the previous subsection as that dealt with the case where the depth $d$ is a fixed constant and not allowed to grow with $n$.

To prove this result, we will make use of a similar result proved by \citet{raz_balance} in the context of multilinear circuits.  In particular, Theorem 3.1 from \citet{raz_balance} states, roughly speaking, that for any syntactically multilinear arithmetic circuit over $y = (y_1,...,y_\ell)$ of size $s$ (of \emph{arbitrary} depth) there exists a syntactically multilinear circuit of size $\bigO(\ell^6 s^3)$ and depth $\bigO( \log(\ell) \log(s) )$ which computes the same function.  

Because this depth-reducing transformation doesn't explicitly preserve monotonicity, and deals with multilinear circuits instead of set-multilinear circuits, while using a slightly different definition of arithmetic circuits, it cannot be directly applied to prove an analogous statement for D\&C SPNs.  However, it turns out that the proof contained in \citet{raz_balance} does in fact support a result which doesn't have these issues \citep{raz_personal}.  We state this as the following theorem.

\begin{theorem}(Adapted from Theorem 3.1 of \citet{raz_balance})
Given a monotone syntactically set-multilinear arithmetic circuit (over $y = (y_1,...,y_\ell)$ with sets given by $G_1$,...,$G_n$) of size $s$ and arbitrary depth, there exists a monotone syntactically set-multilinear arithmetic circuit of size $\bigO(s^3)$ and depth $\bigO( \log(n) \log(s) )$ which computes the same function.
\end{theorem}

Note that the size of the constructed circuit is smaller here than in Theorem 3.1 of \citet{raz_balance} because we can avoid the ``homogenization" step required in the original proof, as syntactically set-multilinear arithmetic circuits automatically have this property.

Given this theorem and the equivalence between monotone syntactically set-multilinear arithmetic circuits and D\&C SPNs which was discussed near the end of Section \ref{sec:polys_and_multi}, the following corollary is immediate.
\begin{corollary}
Given a D\&C SPN of size $s$ and arbitrary depth there exists a D\&C SPN of size $\bigO(s^3)$ and depth $\bigO( \log(n) \log(s) )$ which computes the same function.
\end{corollary}

Note that when the size $s$ is a polynomial function of $n$, this depth bound is stated more simply as $\bigO( \log(n)^2 )$.


\section{Circuits vs Formulas}
\label{sec:formula_limitation}

In \citet{SPN_Learn_Gens} the authors gave a learning algorithm for SPNs which produced D\&C SPN formulas.  Recall that formulas are distinguished from more general circuits in that each node has fan-out at most 1.  They are called ``formulas" because they can be written down directly as formula expressions without the need to define temporary variables.

It is worthwhile asking whether this kind of structural restriction limits the expressive efficiency of D\&C SPNs.   

As we show in this section, the answer turns out to be yes, and indeed D\&C SPN formulas are strictly less expressively efficient than more general D\&C SPNs.  This is stated formally as the following theorem:
\begin{theorem}
\label{thm:formula_limitation}
For every input size $n$ there exists a real-valued function $g$ of $x$ such that:
\begin{enumerate}
\item There is a D\&C SPN of size $\bigO(n^{4/3})$ 
which computes $g$, where the SPN's univariate functions $f$ consist only of identity functions.
\item For any choice of the univariate functions $f$, a D\&C SPN formula that computes $g$ must be of size $n^{\Omega(\log(n))}$ (which is super-polynomial in $n$).
\end{enumerate}
\end{theorem}

As in Section \ref{sec:depth_hierarchy}, to prove Theorem \ref{thm:formula_limitation}, we will make use of an analogous result which is a slight modification of one proved by \citet{raz_balance} in the context of multilinear circuits.  This is stated below.

\begin{theorem}(Adapted from Theorem 4.4 of \citet{raz_balance})
\label{thm:raz_formula_limitation}
For every input size $n$ there exists a real-valued function $g$ of $x$ such that:
\begin{enumerate}
\item There is a monotone syntactically multilinear arithmetic circuit over $x$ of size $\bigO(n)$ with nodes of maximum in-degree $\bigO(n^{1/3})$ which computes $g$ for all values of $x$ in $\Real^n$.
\item Any syntactically multilinear arithmetic formula over $x$ that computes $g$ for all values of $x$ in $\Real^n$ must be of size $n^{\Omega(\log(n))}$.
\end{enumerate}
\end{theorem}

As in Section \ref{sec:depth_hierarchy}, the original Theorem 4.4 from \citet{raz_balance} uses a slightly different definition of arithmetic circuits from ours (they do not permit weighted connections), and the constructed circuits are not stated to be monotone.  However we have confirmed with the authors that their result still holds even with our definition, and the circuits constructed in their proof are indeed monotone \citep{raz_personal}. 

When trying to use Theorem \ref{thm:raz_formula_limitation} to prove Theorem \ref{thm:formula_limitation}, we encounter similar obstacles to those encountered in Section \ref{sec:depth_hierarchy}. 
Fortunately, the transformation between decomposable SPNs and multilinear arithmetic circuits (for the case of binary-valued inputs) happens to preserve formula structure.  Thus the ideas discussed in Section \ref{sec:depth_hierarchy} for overcoming these obstacles also apply here.

\section{A Tractable Distribution Separating D\&C SPNs and Other Deep Models}
\label{sec:spanning_tree_lb}



The existence of a D\&C SPN of size $s$ for computing some density function (possibly unnormalized) implies that the corresponding marginal densities can be computed by an $\bigO(s)$ time algorithm.  Thus, it follows that D\&C SPNs cannot efficiently compute densities whose marginals are known to be intractable.  And if we assume the widely believed complexity theoretic conjecture that $P \neq \#P$, such examples are plentiful.

However, it is debatable whether this should be considered a major drawback of D\&C SPNs, since distributions with intractable marginals are unlikely to be learnable using \emph{any} model.  Thus we are left with an important question: can D\&C SPNs efficiently compute any density with tractable marginals?

In \citet{poon2011sum} it was observed that essentially every known model with tractable marginal densities can be viewed as a D\&C SPN, and it was speculated that the answer to this question is yes.

In this section we refute this speculation by giving a counter example.  In particular, we construct a simple distribution $\mathcal{D}$ whose density function and corresponding marginals can be evaluated by efficient algorithms, but which provably cannot be computed by a sub-exponentially sized D\&C SPN of \emph{any} depth.  Notably, this density function can be computed by a Boolean circuit of modest depth and size, and so by the simulation results of \citep{Deep_Expressive_Martens} the distribution can in fact be captured efficiently by various other deep probabilistic models like Deep Boltzmann Machines (DBMs).

Notably, our proof of the lower bound on the size of D\&C SPNs computing this density function will not use any unproven complexity theoretic conjectures, such as $P \neq \#P$.

It is worthwhile considering whether there might be distributions which can be efficiently modeled by D\&C SPNs but not by other deep generative models like DBMs or Contrastive Backprop Networks \citep{contrastiveBP}.  The answer to this question turns out to be no.  

To see this, note that arithmetic circuits can be efficiently approximated by Boolean circuits, and even more efficiently approximated by linear threshold networks (which are a simple type of neural network).  Thus, by the simulations results of \citet{Deep_Expressive_Martens} the aforementioned deep models can efficiently simulate D\&C SPNs of similar depths (up to a reasonable approximation factor).  Here ``efficiently" means ``with a polynomial increase in size", although in practice this polynomial can be of low order, depending on how exactly one decides to simulate the required arithmetic.  For linear threshold networks (and hence also Contrastive Back-prop Nets), very efficient simulations of arithmetic circuits can be performed using the results of \citet{reif}, for example.

\subsection{Constructing the distribution}

To construct the promised distribution over values of $x$ we will view each $x_i$ as an indicator variable for the presence or absence of a particular labeled edge in a subgraph $G_x$ of $K_m$, where $K_m$ denotes the complete graph on $m$ vertices.  In particular, $x_i$ will take the value $1$ if the edge labeled by $i$ is present in $G_x$ and $0$ otherwise.  Note that there are $\binom{m}{2}$ total edges in $K_m$ and so the total input size is $n = \binom{m}{2}$.  

The distribution $\mathcal{D}$ will then be defined simply as the uniform distribution over values of $x$ satisfying the property that $G_x$ is a spanning tree of $K_m$.  We will denote its density function by $d(x)$.

Computing $d(x)$ up to a normalizing constant\footnote{The normalizing constant in this case is given by $Z = m^{m-2}$ by Cayley's Formula.} amounts to deciding if the graph $G_x$ represented by $x$ is indeed a spanning tree of $K_m$, and outputting $1$ if it is, and $0$ otherwise.   
And to decide if the graph $G_x$ is a spanning tree amounts to checking that it is connected, and that it has exactly $m-1$ edges.   

The first problem can be efficiently solved by a Boolean circuit with $\bigO(n)$ gates and depth $\bigO(\log(n))$ using the well-known trick of repeatedly squaring the adjacency matrix.  The second can be solved by adding all of the entries of $x$ together, which can also be done with a Boolean circuit with $\bigO(n)$ gates and depth $\bigO(\log(n))$ \citep{carry_save}.   Due to how neural networks with simple linear threshold nonlinearities can simulate Boolean circuits in a 1-1 manner \citep[e.g][]{parberry}, it follows that such networks of a similar size and depth can compute $d(x)$.  And since linear threshold gates are easily simulated by a few sigmoid or rectified linear units, it follows that neural networks of the same dimensions equipped with such nonlinearities can also compute $d(x)$, or at least approximate it arbitrarily well (see \citet{Deep_Expressive_Martens} for a review of these basic simulation results).

Moreover, by the results of \citet{Deep_Expressive_Martens} we know that any distribution whose density is computable up to a normalization constant by Boolean circuits can be captured, to an arbitrarily small KL divergence, by various deep probabilistic models of similar dimensions.  In particular, these results imply that Deep Boltzmann Machines of size $\bigO(n)$ and Constrastive Backprop Networks of size $\bigO(n)$ and depth $\bigO(\log(n))$ can model the distribution $\mathcal{D}$ to an arbitrary degree of accuracy.  And since we can sample $(n-2)$-length Pr\"{u}fer sequences \citep{prufer} and implement the algorithm for converting these sequences to trees using a threshold network of size $\bigO(n^2)$ and depth $\bigO(n)$ it follows from \citet{Deep_Expressive_Martens} that we can also approximate $\mathcal{D}$ using Sigmoid Belief Networks \citep{radford_SBN} of this size and depth.


%
%
%
%
%
%

While the existence of small circuits for computing $d(x)$ isn't too surprising, it is a somewhat remarkable fact that it is possible to evaluate any marginal of $d(x)$ using an $\bigO(n^{1.19})$-time algorithm.   That is, given a subset $I$ of $\{1,...,n\}$, and associated fixed values of the corresponding variables (i.e. $x_I$), we can compute the sum of $d(x)$ over all possible values of the remaining variables (i.e. $x_{\{1,...n\}\setminus I}$) using an algorithm which runs in time $\bigO(n^{1.19})$.  

To construct this algorithm we first observe that the problem of computing these marginal densities reduces to the problem of counting the number of spanning trees consistent with a given setting of $x_I$ (for a given $I$).  And it turns out that this is a problem we can attack directly by first reducing it to the problem of counting the total number of spanning trees of a certain auxiliary graph derived from $x_I$, and then reducing it to the problem of computing determinants of the Laplacian matrix of this auxiliary graph via an application of generalized version of Kirchoff's famous Matrix Tree Theorem \citep{tutte2001graph}.   This argument is formalized in the proof of the following theorem.

\begin{theorem}
\label{thm:efficientmarginalcomputation}
There exists a $\bigO(n^{1.19})$-time algorithm, which given as input a set $I \subset \{1,...,n\}$ and corresponding fixed values of $x_I$, outputs the number of edge-labeled spanning trees $T$ of $K_m$ which are consistent with those values.
\end{theorem}

\subsection{Main lower bound result}

The main result of this section is stated as follows:
\begin{theorem}
\label{thm:main_lb}
Suppose that $d(x)$ can be approximated arbitrarily well by D\&C SPNs of size $\leq s$ and $m \geq 20$.  Then $s \geq 2^{m/30240}$.
\end{theorem}
By ``approximated arbitrarily well by D\&C SPNs of size $\leq s$" we mean that there is a sequence of D\&C SPNs of size $\leq s$ whose output approaches $d(x)$, where the univariate functions $f$ are allowed to be different for each SPN in the sequence.  Observe that $d(x)$ being computed exactly by a D\&C SPN of size $s$ trivially implies that it can approximated arbitrarily well by D\&C SPNs of size $\leq s$.

Note that the large constant in the denominator of the exponent can likely be lowered substantially with a tighter analysis than the one we will present.  However, for our purposes, we will be content simply to show that the lower bound on $s$ is exponential in $m$ (and hence also in $\sqrt{n}$).

Our strategy for proving Theorem \ref{thm:main_lb} involves two major steps.  In the first we will show that the output polynomial of any D\&C SPN of size $s$ can be ``decomposed" into the sum of $s^2$ ``weak" functions.  We will then extend this result to show that the same is true of any function which can computed as the limit of the outputs of an infinite sequence of D\&C SPNs of size $\leq s$.  This will be Theorem \ref{thm:sequenceofspns}.

In the second step of the proof we will show that in order to express $d(x)$ as the sum of such ``weak" functions, the size $k$ of the sum must be exponentially large in $m$, and thus so must $s$.  This will follow from the fact (which we will show) that each ``weak" function can only be non-zero on a very small fraction of the all the spanning trees of $K_m$ (to avoid being non-zero for a non-spanning tree graph), and so if a sum of them has the property of being non-zero for all of the spanning trees, then that sum must be very large.  This will be Theorem \ref{thm:sum_product_lb}.

Theorem \ref{thm:main_lb} will then follow directly from Theorems \ref{thm:sequenceofspns} and \ref{thm:sum_product_lb}.

\subsection{Decomposing D\&C SPNs}

The following theorem shows how the output polynomial of a D\&C SPN of size $s$ can be ``decomposed" into a sum of $s^2$ non-negatives functions which are ``weak" in the sense that they factorize over two relatively equal-sized partitions of the set of input variables. 

\begin{theorem}
\label{thm:decomposition}
Suppose $\Phi$ is a D\&C SPN over $f$ of size $s$.  Then we have:
\begin{align*}
q_{\Phi} = \sum_{i=1}^k g_i h_i 
\end{align*}
where $k \leq s^2$, and where the $g_i$'s and $h_i$'s are non-negative polynomials in $f$ satisfying the conditions:
\begin{align}
\label{comparablescopeconditions}
  \frac{n}{3} \leq |x_{g_i}|,|x_{h_i}| \leq \frac{2n}{3}, \quad  x_{g_i} \cap x_{h_i} = \emptyset, \quad x_{g_i} \cup x_{h_i} = x
\end{align}
\end{theorem}

It should be noted that this result is similar to an existing one proved by \citet{raz2011multilinear} for monotone multilinear circuits, although we arrived at it independently. 

While Theorem \ref{thm:decomposition} provides a useful characterization of the form of functions which can be computed exactly by a D\&C SPN of size $s$, it doesn't say anything about functions which can only be computed approximately.  To address this, we will strengthen this result by showing that any function which can be approximated arbitrarily well by D\&C SPNs of size $s$ also has a decomposition which is analogous to the one in Theorem \ref{thm:decomposition}.  This is stated as follows.

\begin{theorem}
\label{thm:sequenceofspns}
Suppose $\{\Phi_j\}_{j=1}^{\infty}$ is a sequence of D\&C SPNs of size at most $s$ (where the definitions of the univariate functions $f$ is allowed to be different for each), such that the sequence $\{q_{\Phi_k}\}_1^{\infty}$ of corresponding output polynomials converges pointwise (considered as functions of $x$) to some function $\gamma$ of $x$.  And further suppose that the size of the range of possible values of $x$ is given by some finite $d$. Then we have that $\gamma$ can be written as
\begin{align}
\label{eqn:sum_product_form}
\gamma = \sum_{i=1}^k g_i h_i
\end{align}
where $k \leq s^2$ and $\forall i$, $g_i$ and $h_i$ are real-valued non-negative functions of $y_i$ and $z_i$ (resp.) where $y_i$ and $z_i$ are sub-sets/tuples of the variables in $x$ satisfying $\frac{n}{3} \leq |y_i|, |z_i| \leq\frac{2n}{3}$, $y_i \cap z_i = \emptyset$, $y_i \cup z_i = x$.
\end{theorem}

\subsection{A lower bound on k}

In this section we will show that if $d(x)$ is of the same form of $\gamma(x)$ from eqn.~\ref{eqn:sum_product_form}, then the size of the size $k$ of the sum must grow exponentially with $m$ (and hence $\sqrt{n}$).   In particular, we will prove the following theorem.
\begin{theorem}
\label{thm:sum_product_lb}
Suppose $d(x)$ is of the form from eqn.~\ref{eqn:sum_product_form}, and $m \geq 20$.  Then we must have that $k \geq 2^{m/15120}$.
\end{theorem}

Our strategy to prove this result will be to show that each term in the sum can only be non-zero on an exponentially small fraction of all the spanning trees of $K_m$ (and is thus ``weak").  And since the sum must be non-zero on \emph{all} the spanning trees in order to give $d(x)$, it will follow that $k$ will have to be exponentially large.


We will start with the simple observation that, due to the non-negativity of the $g_i$'s and $h_i$'s, each factored term $g_i h_i$ in the sum $d = \sum_{i=1}^k g_i h_i$ must agree with $d$ wherever $d$ is 0 (i.e. because we have $d(x) \geq g_i(y_i) h_i(z_i)$ for each $i$).  And in particular, for each value of $x$ with $d(x) = 0$, either $g_i(y_i)$ or $h_i(z_i)$ must be 0.

Intuitively, this is a very stringent requirement.  As an analogy, we can think of each factor ($g_i$ or $h_i$) as ``seeing" roughly half the input edges, and voting ``yes, I think this is a spanning tree", or ``no, I don't think this is a spanning tree" by outputting either a value $>0$ for ``yes" or $0$ for ``no", with tie votes always going ``no".  The requirement can thus be stated as: ``each pair of factors is never allowed to reach an incorrect `yes' decision".

Despite both factors in each pair being arbitrary functions of their respective inputs (at least in principle), each only ``sees" the status of roughly half the edges in the input graph, and so cannot say much about whether the entire graph actually is a spanning tree.  While some potential cycles might be entirely visible from the part of the graph visible to one of the factors, this will not be true of most potential cycles.  Thus, to avoid ever producing an incorrect ``yes" decision, the factors are forced to vote using a very conservative strategy which will favor ``no".  

The remainder of this section is devoted to formalizing this argument by essentially characterizing this conservative voting strategy and showing that it leads to a situation where only a very small fraction of all of the possible spanning trees of $K_m$ can receive two ``yes" votes.  


\begin{lemma}
\label{lemma:main_lb_lemma}
Suppose $g(y)$ and $h(z)$ are real-valued non-negative functions of the same form as those described in eqn.~\ref{eqn:sum_product_form}, and that for any value of $x$, $d(x) = 0$ implies $g(y) = 0$ or $h(z) = 0$.  Define $P = |\{ x \in \{0,1\}^m : d(x) = 1 \mbox{ and } g(y) h(z) > 0 \}|$ and $Z = |\{ x \in \{0,1\}^m : d(x) = 1\}| = m^{m-2}$. Then for $m \geq 20$ we have
\begin{align*}
\frac{P}{Z} \leq \frac{1}{2^{m/15120}}
\end{align*}
\end{lemma}

It is not hard to see that this lemma will immediately imply Theorem \ref{thm:sum_product_lb}.  In particular, provided that $d(x) = 0$ implies that each term in the sum is $0$, we have each term can be non-zero on at most a proportion $\frac{1}{2^{m/15120}}$ of the values of $x$ for which $d(x) = 1$, and thus the entire sum can be non-zero on at most a proportion at most $\frac{k}{2^{m/15120}}$.  Thus we must have that $k \frac{1}{2^{m/15120}} \geq 1$, i.e. $k \geq 2^{m/15120}$.

The rest of this section will be devoted to the proof of Lemma \ref{lemma:main_lb_lemma}, which begins as follows.

Suppose we are given such a $g$ and $h$.  We will color all of the edges of the complete graph $K_m$ as red or blue according to whether they correspond to input variables from $y$ or $z$ (resp.).  

We define a ``triangle" of a graph to be any complete subgraph on 3 vertices.  $K_m$ has $\binom{m}{3}$ triangles total since it is fully connected.   After coloring $K_m$, each triangle is either monochromatic (having edges with only one color), or dichromatic, having 2 edges of one color and 1 edge of the other color.   We will refer to these dichromatic triangles as ``constraint triangles", for reasons which will soon become clear.  

Clearly any graph $G_x$ which is a spanning tree of $K_m$ can't contain any triangles, as these are simple examples of cycles.  And determining whether $G_x$ contains all 3 edges of a given constraint triangle is impossible for $g$ or $h$ by themselves, since neither of them gets to see the status of all 3 edges.  Because of this, $g$ and $h$  must jointly employ one of several very conservative strategies with regards to each constraint triangle in order to avoid deciding ``yes" for some graph containing said triangle.  In particular, we can show that either $g$ must always vote `no' whenever all of the red edges of the triangle are present in the input graph $G_x$, or $h$ must vote ``no" whenever all of the blues edges of the triangle are present in $G_x$.

This is formalized in the following proposition.

\begin{proposition}
\label{prop:dichotomy}
Let $a,b$ and $c$ be edges that form a constraint triangle in $K_n$. Suppose that $a$ and $b$ are both of a different color from $c$.  

Then one of the following two properties holds:
\begin{itemize}
\item $g(y)h(z) = 0$ for all values of $x$ such that $G_x$ contains both $a$ and $b$
\item $g(y)h(z) = 0$ for all values of $x$ such that $G_x$ contains $c$
\end{itemize}
\end{proposition}

Thus we can see that each constraint triangle over edges $a$, $b$, and $c$ in $K_m$ gives rise to distinct constraint which must be obeyed by any graph $G_x$ for which $g(y)h(z) > 0$.  These are each one of two basic forms:
\begin{enumerate}
\item $G_x$ doesn't contain both $a$ and $b$
\item $G_x$ doesn't contain $c$
\end{enumerate}

We now give a lower bound on the number of constraint triangles (i.e. the number of dichromatic triangles) in $K_m$ as a function of the number edges of each color.  
\begin{lemma}
\label{lemma:dichromatictriangles}
Given any coloring of the complete graph $K_m$ with $m \geq 20$ which has $r$ red edges and $n-r$ blue edges (recall $n = \binom{m}{2}$ is the total number of edges), for $n/3 \leq r \leq 2n/3$, the total number of dichromatic triangles is lower bounded by $m^3/60$.
\end{lemma}
Our proof of the above lemma makes use of a known upper bound of the number of triangles in an arbitrary graph due to \citet{fisher1989lower}.

As the choice of $y$ and $z$ implies the hypothesis $n/3 \leq r \leq 2n/3$ we can apply this lemma to conclude that there are at least $m^3 / 60$ constraint triangles, and thus any graph $G_x$ for which $g(y)h(z) > 0$ must obey $m^3 / 60$ distinct constraints of the forms given above.

It remains to show that the requirement of obeying $m^3 / 60$ such constraints limits the number of graphs $G_x$ for which $g(y)h(z) > 0$ to be an exponentially small proportion of the total.  Our strategy for doing this will be as follows.  We will consider a randomized procedure \citep[due to][]{aldous1990random} that samples uniformly from the set of all spanning trees of $K_m$ by performing a type of random walk on $K_m$, adding an edge from the previous vertex whenever it visits a previously unvisited vertex.  We will then show that the sequence of vertices produced by this random walk will, with very high probability, contain a length-3 subsequence which implies that the sampled tree violates at least one of the constraints.

This argument is formalized in the proof of the following lemma.

\begin{lemma}
\label{lemma:fractionoftrees}
Suppose we are given $C$ distinct constraints which are each one of the two forms discussed above.  Then, of all the spanning trees of $K_m$, a proportion of at most
\begin{align*}
\left( 1 - \frac{C}{m^3} \right )^{C/(6m^2)}
\end{align*}
of them obey all of the constraints.
\end{lemma}

As we have $C \geq m^3 / 60$ constraints, this lemma tells us that the proportion of spanning trees $G_x$ for which $g(y)h(z) > 0$ is upper bounded by
\begin{align*}
\left( 1 - 1/60 \right )^{m/360} = \frac{1}{2^{-\log_2(1-1/60) m /360 }} \leq \frac{1}{2^{(1/42) m/360}} = \frac{1}{2^{m/15120}}
\end{align*}

This finally proves Lemma \ref{lemma:main_lb_lemma}, and thus Theorem \ref{thm:sum_product_lb}.

\section{Discussion and future directions}

We have shown that there are tractable distributions which D\&C SPNs cannot efficient capture, but other deep models can.  However, our separating distribution $\mathcal{D}$, which is the uniform distribution over adjacency matrices of spanning trees of the complete graph, is a somewhat ``complicated" one, and seems to require $\log(n)$ depth to be efficiently captured by other deep models.  Some questions worth considering are:
\begin{itemize}
	\item Is a distribution like $\mathcal{D}$ learnable by other deep models in practice?
	\item Is there a simpler example than $\mathcal{D}$ of a tractable separating distribution?
	\item Can we extend D\&C SPNs in some natural way that would allow them to capture distributions like $\mathcal{D}$?
	\item Should we care that D\&C SPNs have this limitation, or are most ``natural" distributions that we might want to model with D\&C SPNs of a fundamentally different character than $\mathcal{D}$?
\end{itemize}

Far from showing that D\&C SPNs are uninteresting, we feel that this paper has established that they are a very attractive objects for theoretical analysis.  While the D\&C conditions limit SPNs, they also make it possible for us to prove much stronger statements about them than we otherwise could.  

Indeed, it is worth underlining the point that the results we have proved about the expressive efficiency of D\&C SPNs are \emph{much} stronger and more thorough than results available for other deep models.  This is likely owed to the intrinsically tractable nature of D\&C SPNs, which makes them amenable to analysis using known mathematical methods, avoiding the various proof barriers that exist for more general circuits.

One aspect of SPNs which we have not touched on in this work is their learnability.  It is strongly believed that for conditional models like neural networks, which are capable of efficiently simulating Boolean circuits, learning is hard in general \citep{shai}.  However, D\&C SPNs don't seem to fall into this category, and to the best of our knowledge, it is still an open question as to whether there is a provably effective and efficient learning algorithm for them.  It seems likely that the ``tractable" nature of D\&C SPNs, which has allowed us to prove so many strong statements about their expressive efficiency, might also make it possible to prove strong statements about their learnability.

\subsubsection*{Acknowledgments}
The authors would like to thank Ran Raz for his helpful discussions regarding multilinear circuits.  James Martens was supported by a Google Fellowship.

\bibliographystyle{plainnat}
\bibliography{SPN_bibliography}

\newpage
\appendix


\section{Proofs for Section \ref{sec:definitions}}

\begin{proof}[Proof of Proposition \ref{prop:normalize}]

We will sketch a proof of this result by describing a simple procedure to transform $\Phi$ into $\Phi'$.

This procedure starts with the leaf nodes and then proceeds up the network, processing a node only once all of its children have been processed.   After being processed, a node will have the property that it computes a normalized density, as will all of its descendant nodes.  

To process a node $u$, we first compute the normalizing constant $Z_u$ of its associated density. If it's a sum node, we divide its incoming weights by $Z_u$, and if it's a leaf node computing some univariate funciton of an $x_i$, we transform this function by dividing it by $Z_u$.  Clearly this results in $u$ computing a normalized density.  Note that processing a product node is trivial since it will be automatically normalized as soon its children are (which follows from decomposability). 

After performing this normalization, the effect on subsequent computations performed by ancestor nodes of $u$ is described as follows.  For every sum node $v$ which is a parent of $u$, the effect is equivalent to dividing the weight on edge $(u,v)$ by $Z_u$.    And for every product node $v$ which is a parent of $u$, the effect is to divide its output by $Z_u$, which affects subsequent ancestor nodes of $v$ by applying this analysis recursively.  The recursive analysis fails only once we reach the root node $r$, in which case the effect is to divide the output of the SPN by the constant $Z_u$, which won't change the SPN's normalized density function (or distribution).

Thus we can compensate for the normalization and leave the distribution associated with the SPN unchanged by multiplying certain incoming edge weights of certain ancestor sum nodes of $u$ by $Z_u$ (as specified by the above analysis).  This is the second step of processing a node $u$.

Note that because we only process a sum node $u$ after its children have been processed and thus each compute normalized densities themselves, $Z_u$ is given simply by the sum of the weights of the incoming edges to $u$.

\end{proof}

\begin{proof}[Proof of Fact \ref{fact:pos_coeff}]

Denote by $r$ the root/output node of $\Phi$, and $\{u_i\}_i$ its child nodes.

This proof is a simple induction on the depth $d$ of $\Phi$.

Suppose the claim is true for $\Phi$'s of depth $d$, and consider a $\Phi$ of depth $d+1$.  

Each $\Phi_u$ is clearly a monotone arithmetic circuit since $\Phi$ is, and so by induction we have that the coefficients of each of the $q_u$'s are non-negative. 

If the root node $r$ is a product node then the coefficients of monomials in $q_\Phi$ are given by sums of products over coefficients of monomials from the $q_u$'s, and are thus non-negative.  And if the root node $r$ is a sum node, then since the weights are non-negative, the coefficients of the monomials in $q_\Phi$ are non-negatively weighted linear combinations of the coefficients of monomials in the $q_u$'s, and are thus also non-negative.

Thus $q_\Phi$ is indeed a non-negative polynomial.

The base case is trivial since a depth 1 circuit consists of a single node which computes either a non-negative constant or the value of a single variable, and both of these correspond to a non-negative polynomial in the input variables.

\end{proof}

\section{Proofs for Section \ref{sec:analysis}}

\begin{proof}[Proof of Proposition \ref{prop:completeness_transform}]

We will form $\Phi'$ from $\Phi$ via a modification procedure which is described as follows.  

Given a sum node $u$ in $\Phi$, with dependency-scope $x_u$, for each child $v$ of $u$, if $v$'s dependency-scope $x_v$ doesn't coincide with $x_u$ (i.e. $x_v \neq x_u$, where $\neq$ means inequality of sets), we add a new product node to $\Phi$ and make this node a child of $u$ in place of $v$.  This product node's children consist of the original $v$ and a set of nodes computing univariate functions of each $x_i \in x_u \setminus x_v$ which are each constant and equal to $1$ for all values of $x_i$.  Note that the dependency scope of this new node is thus $x_v \cup x_u \setminus x_v = x_u$. 

Note that these modifications enlarge the set/tuple $f$ of univariate functions to one which we will denote by $f'$.

Also note that these modifications add $k \in \bigO(s^2)$ product nodes, and an additional $n$ nodes for computing the constant univariate functions of each $x_i$.  Thus the size of $\Phi'$ is given by $s + n + k \in \bigO(s^2)$.

To see that $q_{\Phi'}(f'(x)) = q_\Phi(f(x))$ for all values of $x$, observe that each node we add computes the same function as the node it replaces
, as long as the network is evaluated in the standard way for a particular value of $x$ (i.e. where the input to the underlying arithmetic circuit is $f(x)$).  This is because the node computes the product between the aforementioned replaced node and a set of nodes which will always output $1$ when evaluated for a particular value of $x$.

To see that $\Phi'$ is complete, consider any sum node $u$ in $\Phi'$.  By construction, a given child node $z$ of $u$ is either the original child node $v$ from $\Phi$ in the case where $x_v = x_u$ (as sets), or is one of the above constructed product nodes, and thus has dependency-scope $x_u$.  Thus $u$ satisfies the condition required for completeness.  And since $u$ was general we can conclude that $\Phi'$ is complete.

Finally, it remains to show that $\Phi'$ is decomposable if $\Phi$ is.  But this is easy, since we didn't modify any of the dependency-scopes of any existing nodes, and the new product nodes we added had children whose dependency-scopes were disjoint by construction.

\end{proof}

\vspace{0.7in}

\begin{proof}[Proof of Theorem \ref{thm:strongvalidity_outputpoly}]
\mbox{}  

$(\Longleftarrow)$

We will first show the reverse direction, that set-multilinearity of the output polynomial of $\Phi$ implies that $\Phi$ is strongly valid.

Suppose $q_{\Phi}$ is set-multilinear.  Then we have that by definition $q_{\Phi}$ is a polynomial of the form
\begin{align*}
\sum_{k = 1}^t c_k \prod_{i=1}^{n} f_{i,j_{i,k}}(x_i)
\end{align*}
for some coefficients $c_k$ and indices $j_{i,k}$.

Further suppose that $I = \{i_1, i_2, ..., i_\ell\} \subseteq [n]$ and $S_{i_1} \subseteq R_{i_1}, S_{i_2} \subseteq R_{i_2}, ..., S_{i_\ell} \subseteq R_{i_\ell}$.

Then for any choice of the $f_{i,j}$'s and $x_{[n] \setminus I}$ we have
\begin{align*}
\int_{S_{i_1} \times S_{i_2} \times ... \times S_{i_\ell}} &  q_{\Phi}(f(x)) \: \mathrm{d} M_I(x_I) = \int_{S_{i_1} \times S_{i_2} \times ... \times S_{i_\ell}} \sum_{k = 1}^t c_k f_{i,j_{i,k}}(x_i) \: \mathrm{d} M_I(x_I) \\
&= \sum_{k = 1}^t c_k \int_{S_{i_1} \times S_{i_2} \times ... \times S_{i_\ell}} f_{i,j_{i,k}}(x_i) \mathrm{d} M_I(x_I) \\
&= \sum_{k = 1}^t c_k \int_{S_{i_1}} \int_{S_{i_2}} \int_{S_{i_3}} ...\int_{S_{i_\ell}} \prod_{i=1}^{n} f_{i,j_{i,k}}(x_i) \: \mathrm{d} M_{i_\ell}(x_{i_\ell}) ... \: \mathrm{d} M_{i_2}(x_{i_2}) \: \mathrm{d} M_{i_1}(x_{i_1}) \\
&= \sum_{k = 1}^t c_k \prod_{i \in [n] \setminus I} f_{i,j_{i,k}}(x_i) \prod_{i \in I} \int_{S_i} f_{i,j_{i,k}}(x_i) \: \mathrm{d} M_i (x_i) \\
&= q_{\Phi} A_I( S_I, x_{[n] \setminus I})
\end{align*}
and so $\Phi$ is valid for these choices.  

Thus the reverse direction holds.

$(\Longrightarrow)$

Now suppose that $\Phi$ is strongly valid and suppose by contradiction that $q_{\Phi}$ is not set-multilinear.

One way that this can happen (``case 1") is if there is some monomial $m$ in $q_{\Phi}$ such that for some $i_0 \in [n]$, $m$ has a factor of the form $f_{i_0,j_1} f_{i_0,j_2}$ depending on the same $x_{i_0}$ (it might be the case that $j_1 = j_2$).  The other way this can happen (``case 2") is $m$ has no factors which are functions depending on some $x_{i_0}$ in $q_{\Phi}$'s dependency-scope.

For each $i \neq i_0$ we will choose the univariate function $f_{i,j}$ for each $j$ so that $f_{i,j}(x_i)$ is non-zero for at least one particular value of $x_i$ in $R_i$ and $\int_{R_i} f_{i,j}(x_i) \: \mathrm{d} M_i (x_i)$ is finite.  

We can use the fact that $x_{i_0}$ is non-trivial to find two disjoint subsets $S$ and $T$ of $R_{i_0}$ which both have finite and non-zero measure under $M_i$.  For each $j$, choose $f_{i_0,j}$ so that $\int_S f_{i_0,j}(x_{i_0}) \: \mathrm{d} M_{i_0} (x_{i_0}) = \int_T f_{i_0,j}(x_{i_0}) \: \mathrm{d} M_{i_0} (x_{i_0}) = b$, for a positive valued variable $b$.

As remarked before, we select values of $x_{[n] \setminus \{i_0\}}$ so that each $f_{i,j}(x_i)$ is positive.   Fix such values.  Then it is not hard to see that $q_{\Phi} A_{\{i_0\}}( S, x_{[n] \setminus \{i_0\} })$ and $q_{\Phi} A_{\{i_0\}}( T, x_{[n] \setminus \{i_0\} })$ can be viewed as a polynomial functions of $b$.  Moreover, they are both equal to the same polynomial, which we will denote $g(b)$.   Combining the positivity of each $f_{i,j}(x_i)$ with Fact \ref{fact:pos_coeff} (which says that all coefficients on monomial terms in $q_{\Phi}$ are positive) we have that each monomial term in $q_{\Phi}$, when viewed as a monomial term in $g(b)$, has a positive coefficient.  Collecting like terms in $g$ we have that since all coefficients are positive, there can be no cancellation of terms, and thus there is some monomial in $g$ of the form $b^k$ for $k \geq 2$ (case 1) or $k = 0$ (case 2).

Similarly, for the same fixed values of $x_{[n] \setminus \{i_0\}}$ we have that $q_{\Phi} A_{\{i_0\}}( S \cup T, x_{[n] \setminus \{i_0\} })$ can be viewed as a polynomial function of $b$.  Moreover, using the fact that
\begin{align*}
\int_{S \cup T} f_{{i_0},j}(x_{i_0}) \mathrm{d} M_{i_0}(x_{i_0}) = \int_{S} f_{{i_0},j}(x_{i_0}) \mathrm{d} M_{i_0}(x_{i_0}) + \int_{T} f_{{i_0},j}(x_{i_0}) \mathrm{d} M_{i_0}(x_{i_0}) = 2b
\end{align*}
this polynomial turns out to be $g(2b)$. 

Because $\Phi$ is strongly valid, we have for all values of $b$
\begin{align*}
g(b) + g(b) &= q_{\Phi} A_{\{i_0\}}( S, x_{[n] \setminus \{i_0\} }) + q_{\Phi} A_{\{i_0\}}( T, x_{[n] \setminus \{i_0\}}) \\
&= \int_S q_{\Phi}(f(x)) \: \mathrm{d} M_{i_0}(x_{i_0}) + \int_T q_{\Phi}(f(x)) \: \mathrm{d} M_{i_0}(x_{i_0}) \\
&= \int_{S \cup T} q_{\Phi}(f(x)) \: \mathrm{d} M_{i_0}(x_{i_0}) \\
&= q_{\Phi} A_{\{i_0\}}( S \cup T, x_{[n] \setminus \{i_0\} }) = g(2b)
\end{align*}

Thus we have that $2g(b) = g(2b)$ for all positive values of $b$, or in other words, $h(b) = 2g(b) - g(2b) = 0$.  Since $h(b)$ is a finite degree polynomial which is 0 for all positive $b$, it must be the zero polynomial (non-zero polynomials can only have finitely many roots).  Thus it has no monomials.  But if $c$ is the coefficient associated with the monomial $b^k$ in $g$, it is easy to see that the coefficient of $b^k$ is $2c - 2^kc = c(2 - 2^k)$ in $h$.  But because $k \neq 1$, this is clearly non-zero, a contradiction.

Thus our assumption that $q_{\Phi}$ was not set-multilinear was incorrect.

\end{proof}

\begin{proof}[Proof of Lemma \ref{lemma:non-degen}]

\leavevmode
\begin{enumerate}
\item This is self-evident from the definition of non-degeneracy.

\item If $r$ is a product node, $q_r$ is given by the product of the $q_{u_i}$'s.  Expanding out this product yields a polynomial expression whose monomial terms (monomial times coefficients) have monomials given by every possible product formed by multiplying together one monomial from each of the $q_{u_i}$'s.  Collecting like terms, we have that the coefficients associated with each of the above described monomials are given by sums of products of the coefficients associated with the monomials from the $q_{u_i}$'s.  Since the coefficients associated with the monomials from the $q_{u_i}$'s are all positive (by Fact \ref{fact:pos_coeff}), so are the coefficients associated with each of the above described monomials in the product expansion, and thus they all appear in $q_r$.

\item If $r$ is a sum node, $q_r$ is given by a weighted sum of the $q_{u_i}$'s.  Performing this weighted sum yields a polynomial expression whose monomial terms have monomials given by taking the union over all of the monomials in the $q_{u_i}$'s.  Collecting like terms, we have that the coefficients associated with each of the above described monomials are given by weighted sums of the coefficients associated with monomials from the $q_{u_i}$'s.  Since the coefficients associated with monomials from the $q_{u_i}$'s are all positive (by Fact \ref{fact:pos_coeff}), and since by non-degeneracy the weights are all positive, so are the coefficients associated with each of the above described monomials, and thus they all appear in $q_r$.

\item We will prove this statement by induction on the depth of $\Phi$.

Suppose that it is true for depth $d$ and consider the case where $\Phi$ is of depth $d+1$.

In the case where $r$ is a product node, we observe that $q_r$ is the product of the $q_{u_i}$'s.  Since each of the $q_{u_i}$'s are non-zero by the inductive hypothesis, it follows that $q_r$ is non-zero (in general, the product of non-zero polynomials is a non-zero polynomial).

In the case where $r$ is a sum node, we have that by the inductive hypothesis that $q_{u_1}$ is not the zero polynomial and so it contains some monomial $m$.  By Part 3, this is also a monomial of $q$, and so $q$ is not the zero polynomial.

For the base case, where the circuit is a single node, it can either be a constant node or a node labeled with some element of $y$.  Because $\Phi$ is non-degenerate, if it is a constant node we know that the constant it computes is non-zero, and thus it doesn't compute the zero polynomial.  If it is labeled with an element of $y$, then it computes the polynomial given by that element, which is clearly not the zero polynomial.

\item Observe that this statement is equivalent to the following one:  for every member $G_k$ of $r$'s set-scope, there exists some monomial in $q_r$ which has an element of $G_k$ as a factor.

We will prove that this restatement is true by induction on the depth of $\Phi$.

Suppose that it is true for depth $d$ and consider the case where $\Phi$ is of depth $d+1$.

Note that the set-scope of $r$ is by definition the union of the set-scopes of the $u_i$'s.  

Suppose $G_k$ is some member of the set-scope of $r$.  Then it is a member of the set-scope of some $u_i$.  By the inductive hypothesis this means that there is some monomial $m$ in $q_{u_i}$ which has an element of $G_k$ as a factor.  Also, by applying Part 4 to each of the $u_j$'s we have that none of the $q_{u_j}$'s are the zero polynomial, and thus each of them have at least one monomial.

In the case where $r$ is a product node, we can apply Part 2 to $\Phi$, so that there is a monomial in $q_r$ which has $m$ as a factor, and thus an element of $G_k$ as a factor.

In the case where $r$ is a sum node, we can apply Part 3 to $\Phi$, so that $m$ is a monomial in $q_r$ (which has an element of $G_k$ as a factor).

Thus the result holds for $\Phi$.

The base case is straightforward.

\end{enumerate}
\end{proof}

\begin{proof}[Proof of Theorem \ref{theorem:sml_output_equiv_setmulti}]

Denote by $r$ the root/output node of $\Phi$, and $\{u_i\}_i$ its child nodes.

$(\Longrightarrow)$

For the forward direction we will give a proof by induction on the depth of $\Phi$. 

Suppose the claim is true for $\Phi$'s of depth $d$, and consider a $\Phi$ of depth $d+1$.

The first case we consider is where the root node $r$ is a product node.  

Suppose by contradiction that the set-scopes of some $u_i$ and $u_j$ overlap.  It follows by Part 5 of Lemma \ref{lemma:non-degen} that the set-scopes of the polynomials $q_{u_i}$ and $q_{u_j}$ overlap.  In other words, there are monomials $m_1$ in $q_{u_i}$ and $m_2$ in $q_{u_j}$ which both have an element of the same $G_k$ as a factor.   By Part 4 of Lemma \ref{lemma:non-degen} we have that the other $q_{u_\ell}$'s are not zero polynomials and thus each of them has at least one monomial.  Taking the product of these monomials with $m_1$ and $m_2$ yields a monomial which has $m_1 m_2$ as a factor, and which is in $q_r$ by Part 2 of Lemma \ref{lemma:non-degen}.  This violates the set-multilinearity of $q_r$, since said monomial has two or more elements of $G_k$ as factors (possibly the same element twice).

Thus we have that the set-scope of the $u_i$'s are disjoint, which is the condition that $r$, being a product node, must satisfy in order for $\Phi$ to be syntactically set-multilinear.

Thus, to show that the $\Phi$ is syntactically set-multilinear, it remains to show that each of the $\Phi_{u_i}$'s are also syntactically set-multilinear arithmetic circuits.  And by the inductive hypothesis this amounts to establishing that their output polynomials (the $q_{u_i}$'s) are set-multilinear.

Suppose by contradiction that this is not the case, and that some $q_{u_i}$ is not set-multilinear.  The first way this can happen is if some monomial $m_1$ in $q_{u_i}$ has 2 distinct factors which are members of the same $G_k$.  By Part 4 of Lemma \ref{lemma:non-degen}, we have that each of the other $q_{u_j}$'s are not zero polynomials and thus each contain at least one monomial.  Combining this fact with Part 2 of Lemma \ref{lemma:non-degen} we thus have that $q_r$ contains at least one monomial which is of the form $m_1 m_2$, where $m_2$ is the product of monomials from the other $q_{u_j}$'s.  This monomial has 2 distinct factors which are members of $G_k$ (since $m_1$ does), which contradicts the set-multilinearity of $q_r$.  The second way that some $q_{u_i}$ can fail to be set-multilinear is if there is some monomial $m_1$ in $q_{u_i}$ such that none of the elements of $G_k$ are a factor of $m_1$, for some $G_k$ in $q_{u_i}$'s set-scope.  As in the previous case this means that there is a monomial in $q_i$ which of the form $m_1 m_2$ where $m_2$ is the product of monomials from the other $q_{u_j}$'s.  Since the set-scopes of the other $q_{u_j}$'s are disjoint from $q_{u_i}$'s set-scope, this means that $m_1 m_2$ doesn't have any element of $G_k$ as a factor.  But this contradicts the set-multilinearity of $q_r$, since $q_r$'s set-scope contains $G_k$ (because, by Part 5 of Lemma \ref{lemma:non-degen}, it is equal to $r$'s set-scope, which contains $u_i$'s set-scope, which itself is equal to $q_{u_i}$'s set-scope).

The second case we consider is where $r$ is a sum node. 

Suppose by contradiction that the set-scopes of some $u_i$ and $u_j$ are distinct.   By Part 5 of Lemma \ref{lemma:non-degen} it follows then that the set-scopes of $q_{u_i}$ and $q_{u_j}$ are distinct.  Let $G_k$ be a member of the set-scope of $q_{u_i}$ which is not in the set-scope of $q_{u_j}$.  By Part 4 of Lemma \ref{lemma:non-degen} we know that $q_{u_j}$ is not the zero polynomial and so there is some monomial $m$ in $q_{u_j}$ (which doesn't have any element of $G_k$ as a factor).  And thus $m$ is also a monomial in $q_r$ by Part 3 of Lemma \ref{lemma:non-degen}.  But, similarly to the product node case, this contradicts the set-multilinearity of $q_r$.

Thus we have that the set-scope of the $u_i$'s are all identical, which is necessary condition that the sum node $r$ must satisfy in order for $\Phi$ to be syntactically set-multilinear.

Thus, to show that $\Phi$ is syntactically set-multilinear, it remains to show that each of the $\Phi_{u_i}$'s is syntactically set-multilinear.  By the inductive hypothesis this amounts to establishing that their output polynomials (the $q_{u_i}$'s) are set-multilinear.

Suppose by contradiction that this is not the case, and that some $q_{u_i}$ is not set-multilinear.  The first way this can happen is if a monomial $m$ in $q_{u_i}$'s has 2 distinct factors which are elements of the same $G_k$ (might be the same element twice).  But then by Part 3 of Lemma \ref{lemma:non-degen} $m$ must be in $q_r$, which contradicts the set-multilinearity of $q_r$.  The second way that some $q_{u_i}$ can fail to be set-multilinear is if there is some monomial $m$ in $q_{u_i}$ for which none of the elements of $G_k$ are a factor, for a $G_k$ in $q_{u_i}$'s set-scope.  As before, this contradict the set-multilinearity of $q_r$.

The base case, where $\Phi$ has a single node (which must be a node labeled with an element of $y$ or a constant node), is simple to verify.

Thus by induction we have that the forward direction of the theorem's statement holds.

$(\Longleftarrow)$

For the reverse direction we will give a similar proof by induction on the depth.

Suppose the claim is true for $\Phi$'s of depth $d$, and consider a $\Phi$ of depth $d+1$.

Since $\Phi$ is syntactically set-multilinear, each of the $\Phi_{u_i}$'s are as well.   So by the inductive hypothesis we have that the $q_{u_i}$'s are all set-multilinear.  

Consider the case where the root node $r$ is a product node.  By the  syntactic set-multilinearity of $\Phi$ we have that set-scopes of the $u_i$'s are pairwise disjoint, and by Part 5 of Lemma \ref{lemma:non-degen} we have that these are equal to the set-scopes of their respective output polynomials (the $q_{u_i}$'s), and so these are disjoint as well.  Because the $q_{u_i}$'s are set-multilinear, we have that their monomials consist of products between elements of the members of their set-scope, one for each such member.  By Part 2 of Lemma \ref{lemma:non-degen}, $q_r$'s monomials are given by taking every possible product formed by taking exactly one monomial from each of the $q_{u_i}$'s (by Part 2 of Lemma \ref{lemma:non-degen}), and thus these monomials are each a product over elements of the $G_k$'s, with exactly one element from each $G_k$ in the disjoint union of the set-scopes of the $q_{u_i}$'s (which is equal to the union of the set-scope of the $u_i$'s, which is equal to the set-scope of $r$, which is itself equal to the set-scope of $q_r$ [by Part 5 of Lemma \ref{lemma:non-degen}]).  Thus $q_r$ is set-multilinear. 

Now consider the case where $r$ is a sum node.  By the syntactic set-multilinearity of $\Phi$ we have that set-scopes of the $u_i$'s are all equal to some set $S$, and by Part 5 of Lemma \ref{lemma:non-degen} we have that these are equal to the set-scopes of their respective output polynomials (the $q_{u_i}$'s), and so these are all equal to $S$ as well.  Because the $q_{u_i}$'s are set-multilinear, we have that their monomials consist of products between elements of the members $G_k$ of their set-scope, one for each such $G_k$.  By Part 3 of Lemma \ref{lemma:non-degen} we have that the set of monomials appearing in $q_r$ is given by the union over the sets of monomials appearing in each of the $q_{u_i}$'s, and thus these monomials are each a product over elements of the $G_k$'s, with exactly one element from each $G_k$ in $S$ (which, similarly to before, is equal to the set-scope of $r$). Thus $q_r$ is set-multilinear. 

The base case, where $\Phi$ has a single node (which must be a node labeled with an element of $y$ or a constant node), is simple to verify.

\end{proof}

\begin{proof}[Proof of Theorem \ref{thm:validity_coNP-hard}]

We give a polynomial-time reduction from the co-NP complete problem $\overline{\mathrm{SAT}}$, which is the problem of deciding if a Boolean CNF formula has no satisfying assignment, to validity testing of extended SPNs.

Informally, the idea of the reduction is as follows.  Given a Boolean CNF formula $\psi$ over the set of $0/1$-valued variables in $x$ we will construct a poly-sized SPN over $x$ whose output for a particular value of $x$ will be non-zero if and only if this value represents an assignment to the variables which satisfies $\psi$.  By adding some additional circuitry to the SPN (which will involve a few negative weights) we can ensure when the SPN is evaluated for an input corresponding to an integral over one or more of the $x_i$'s, that it always outputs $0$.  Thus, the SPN will be valid if and only if it outputs $0$ whenever it is evaluated for a particular value of $x$, or in other words, if $\psi$ has no satisfying assignment.

We now proceed with the formal argument.

Let $\psi$ be an input CNF formula over $x$, with associated Boolean circuit $\gamma$, where we use the natural correspondence of $0 = \mathrm{FALSE}$, $1 = \mathrm{TRUE}$.  We will construct an SPN $\Phi$ from $\gamma$ as follows.  We replace each literal node (of the form $x_i$ and $\overline{x_i}$) with a node labeled by either $f_{i,1}$ or $f_{i,2}$, where these functions are defined by $f_{i,1}(x_i) = x_i$ and $f_{i,2}(x_i) = \overline{x_i} = 1-x_i$.  Then we replace each AND node with a product node, and each OR node with a sum node (with weights all equal to $1$).  The domain $R_i$ of each $x_i$ will be $\{0,1\}$ and each $M_i$ will be the standard counting measure.

It is clear from this definition that $\Phi(f(x)) > 0$ if and only if $x$ corresponds to a satisfying assignment of $\psi$.

We then augment $\Phi$ with additional circuitry, so that for any choice of $I \neq \emptyset$ and values of $x_{[n] \setminus I}$ we have $\Phi( A_I( S_I, x_{[n] \setminus I})) = 0$, while the standard evaluation of $\Phi(f(x))$ for particular values of $x$ remains unchanged.

Notice that an input to $\Phi$ of the form $A_I( S_I, x_{[n] \setminus I})$ for $I \neq \emptyset$ can be distinguished from an input of the form $f(x)$, by the condition $f_{i,1} = f_{i,2} = 1$ for some $i$.  Thus, to achieve the required property we can augment $\Phi$ by giving it a new output node which computes the product of the original ouput node and a subcircuit which computes $\prod_i(1 - f_{i,1} f_{i,2})$ in the obvious way.  Note that this requires the use of negative weights.

Clearly this construction of $\Phi$ can be accomplished in time polynomial in the size of $\psi$.  Moreover, it is easy to see that $q_\Phi$ is non-negative valued for all values of $x$ in $\{0,1\}^n$.

Suppose $\psi$ has no satisfying assignment.  Then as constructed, $\Phi(f(x)) = 0$ for all values of $x$, and we also have that $f(A_I(S_I, x_{[n]\setminus I})) = 0$ for all values of $I$ and $S_I$, and thus $\Phi$ is valid.   If, on the other hand, $\psi$ has some satisfying assignment given by some value $x'$ of $x$, then $\Phi(f(x')) > 0$.  Validity of $\Phi$ requires that
\begin{align*}
\Phi(f(x')) + \Phi( f( \overline{x'_1}, x'_{[n] \setminus \{1\}} ) ) = \Phi( A_{\{1\}}( S_1, x'_{[n] \setminus \{1\}}) )
\end{align*}
But due to the construction of $\Phi$, the RHS is $0$, while the LHS is clearly positive since $\Phi(f(x'))$ is positive and $\Phi( f( \overline{x'_1}, x'_{[n] \setminus \{1\}} ) )$ is non-negative.  Thus this requirement of validity is not satisfied by $\Phi$ and so $\Phi$ is not valid.

\end{proof}

\section{Proofs for Section \ref{sec:capabilities}}

\begin{proof}[Proof of Proposition \ref{prop:FPLM_simulation}]

To produce the promised D\&C SPN, we will design an arithmetic circuit which implements the algorithm for running an FPLM in the obvious way using basic matrix-vector arithmetic.  

In particular, for each stage $i$, the $k$-dimensional working vector is represented by a group of $k$ nodes, and the required matrix multiplication of the working vector by the matrix $T_{\pi(i)}(x_{\pi(i)})$ is implemented by $k^2$ product nodes and $k$ sum nodes.  Each such product node has two children: one corresponding to a component of the working vector, and the other being a univariate function of the current $x_{\pi(i)}$, which determines the corresponding entry of the matrix $T_{\pi(i)}(x_{\pi(i)})$.  The output of these product nodes are summed as appropriate by the $k$ sum nodes (each of which takes as input the product nodes corresponding to a single row of the matrix) to produce the representation of the working vector at the next stage.

The initialization of the current working vector by $a$ is implemented with constant-valued nodes, and the inner product of the final working vector with $b$ is implemented by a sum node in the obvious way.

It is not hard to see that this construction produces an SPN which is both decomposable and complete.  In particular, decomposability follows from the fact that product nodes each have two children: one which computes a univariate function which depends on some $x_j$ and one whose dependency-scope clearly doesn't contain $x_j$ (because $x_j$ has not yet been processed in the fixed order in which the FPLM processes its input).

The size of the SPN is clearly $\bigO(k^2 n)$.

\end{proof}

\begin{proof}[Proof of Proposition \ref{prop:FPSSM_simulation}]

We observe that FPSSMs are structurally and functionally identical to FPLMs in every respect (including how the state transitions are determined by the current $x_{\pi(i)}$) except in their definitions of states, and kinds of functions and transformations that are performed on them.

We also observe that an FPSSM's working state $u$ can be represented as a vector $v$ using a $1$-of-$k$ encoding, where $u$ being in state $i$ corresponds to $v = e_i$.  Here, $e_i$ is the $k$-dimensional vector which is $1$ in position $i$, and $0$ elsewhere. 

Moreover, arbitrary state transition functions can be accomplished within this representation by linear transformations.  In particular, a mapping $g$ from $[k]$ to $[k]$ can be implemented in the vector representation by multiplication with the matrix $T$, where
\begin{align*}
T = \begin{bmatrix} e_{g(1)} \: e_{g(2)} \: \hdots \: e_{g(n)} \end{bmatrix}
\end{align*}

And arbitrary mappings $h$ from $[k]$ to $\Real_{\geq 0}$ can realized as the inner product of $v$ with a vector $b \in \Real_{\geq 0}$, given by
\begin{align*}
b = \begin{bmatrix} h(1) \\ h(2) \\ \vdots \\ h(n) \end{bmatrix}
\end{align*}

Thus, using this $1$-of-$k$ encoded vector representation scheme allows us to construct, in the obvious way, a FPLM of dimension $k$ which will behave identically to a given FPSSM of state-size $k$.


\end{proof}

\section{Proofs for Section \ref{sec:depth3separation}}

\begin{proof}[Proof of Theorem \ref{thm:3layer_rank_bound}]

First note that due to decomposability and completeness conditions, the nodes in the second layer of $\Phi$ must compute either weighted sums over univariate functions of the same variable, or products between univariate functions of distinct input variables.  In either case, they compute functions which ``factorize" over the $x_i$'s (in the sense that they are of the form $g(x) = \prod_{i=1}^n h_i(x_i)$ for some $h_i$'s).  

So if the root/output node of $\Phi$ is a sum node we therefore have that $\Phi$ computes a weighted sum over $k$ functions, each of which factorizes over the $x_i$'s.  And if the root/output node is a product node we have that $\Phi$ computes a product between functions with disjoint sets of dependent variables (by decomposability), each of which factorizes over its respective dependent variables, and is therefore a function which factorizes over the $x_i$'s.

So in either case we have that the function computed by $\Phi$ is the weighted sum of $\leq k$ functions ($1$ if $r$ is a product node), each of which factorizes over the $x_i$'s.

Consider a general function $g$ over $x$ which factorizes over the $x_i$'s.  We claim that $M^{A,B}_g$ has a rank of 1.  To see this, note that $g$ can be expressed as the product
\begin{align*}
g(x) = \prod_{i=1}^n h_i(x_i) = \left ( \prod_{i \in A} h_i(x_i) \right) \left ( \prod_{i \in B} h_i(x_i) \right)
\end{align*}
for some $h_i$'s.

Let $v_A$ be the vector consisting of the values of $\prod_{i \in A} h_i(x_i)$ for different values of $x_A$ indexed according to the same order in which we index the rows of $M^{A,B}_g$, and define $v_B$ analogously.
Then we have $M^{A,B}_g = v_A v_B^\top$, which has rank 1.

Since $q_{\Phi}$ is the weighted sum of $k$ functions $g_1,g_2,..., g_k$ each of which factorizes over $x$, we have that $M^{A,B}_{q_{\Phi}(f(x))} = \sum_{i=1}^k w_i M^{A,B}_{g_i}$ for some weights $w_i$.  Then using the subadditivity of rank, we get
\begin{align*}
\rank \left(M^{A,B}_{q_{\Phi}(f(x))} \right) = \rank \left(\sum_{i=1}^k w_i M^{A,B}_{g_i} \right) \leq \sum_{i=1}^k \rank \left (M^{A,B}_{g_i} \right) = \sum_{i=1}^k 1 = k
\end{align*}

\end{proof}

\begin{proof}[Proof of Proposition \ref{prop:equal_upper_bound}]

The construction of the promised D\&C SPN proceeds as follows.

In the second layer we have, for each $i$, a pair of product nodes which compute $x_i x_{i+n/2}$ and $\overline{x_i} \: \overline{x_{i+n/2}}$ in the obvious way, where $\overline{z} \equiv 1-z$ (which is a non-negative valued univariate function of $z$).

In the third layer we have, for each $i$, a sum node which computes $x_i x_{i+n/2} + \overline{x_i} \: \overline{x_{i+n/2}}$ from the outputs of the second layer.  It is easy to see that the $i$-th such node outputs $1$ if $x_i = x_{i+n/2}$, and $0$ otherwise.  

The forth and final layer consists of a single product node which computes the product over all of the nodes in the third layer, i.e. $\prod_i x_i x_{i+n/2} + \overline{x_i} \: \overline{x_{i+n/2}}$.  It is easy to see that this function is $1$ if $x_i = x_{i+n/2}$ for all $i \leq n/2$, and is $0$ otherwise.
\end{proof}

\begin{proof}[Proof of Theorem \ref{thm:rank_approx_equal}]

Since the output of $\Phi$ is proportional to the density function of its associated distribution, we have that there is some $\alpha > 0$ s.t. $\alpha /2 \leq q_\Phi(x) \leq \alpha$ each values of $x$ satisfying $\EQUAL(x) = 1$. 

Note that there are $2^{n/2}$ values of $x$ with $\EQUAL(x) = 1$.

Let $\beta = \sum_{x : \EQUAL(x) = 0} q_{\Phi}(f(x))$.  Then bounding the total probability $\delta$ under $\Phi$'s distribution of the set of values of $x$ with $\EQUAL(x) = 0$ we have
\begin{align*}
\delta = \frac{\sum_{x : \EQUAL(x) = 0} q_{\Phi}(f(x))}{\sum_{x : \EQUAL(x) = 0} q_{\Phi}(f(x)) + \sum_{x : \EQUAL(x) = 1} q_{\Phi}(f(x))} \geq  \frac{\beta}{\beta + 2^{n/2}\alpha}
\end{align*}
And rearranging this we obtain
\begin{align*}
\beta \leq \frac{\delta 2^{n/2} \alpha }{1 - \delta}
\end{align*}

Now consider the matrix $M^{H_1,H_2}_{q_{\Phi}(f(x))}$.  The sum of the values of the off-diagonal entries is $\beta$ by definition, and each of the $2^{n/2}$ diagonal entries differs from $\alpha$ by at most $\frac{\alpha}{2}$.  Thus we have that $M^{H_1,H_2}_{q_{\Phi}(f(x))} = \alpha(I + D)$ where $D$ is a matrix satisfying
\begin{align*}
\sum_{i,j} |[D]_{i,j}| \leq \frac{\beta + 2^{n/2} \frac{\alpha}{2} }{\alpha} = \frac{\beta}{\alpha} + \frac{1}{2} 2^{n/2} \leq \frac{\delta 2^{n/2}}{1 - \delta} + \frac{1}{2} 2^{n/2} = 2^{n/2-1} \frac{2\delta + 1-\delta}{1-\delta} = 2^{n/2-1} \frac{1+\delta}{1-\delta}
\end{align*}

Applying Lemma \ref{lemma:rank_perturb_lemma}, we have that
\begin{align*}
\rank(I + D) \geq 2^{n/2}/2 - \left( 2^{n/2-1} \frac{1+\delta}{1-\delta} \right) / 2 = 2^{n/2-2}\left(2 - \frac{1+\delta}{1-\delta} \right)
\end{align*}
Plugging in $\delta \leq 1/4$ we have $\rank \left (M^{H_1,H_2}_{q_{\Phi}(f(x))} \right) = \rank(I + D) \geq 2^{n/2 - 2} / 3$.  The result then follows from Theorem \ref{thm:3layer_rank_bound}.

\end{proof}


\begin{proof}[Proof of Lemma \ref{lemma:rank_perturb_lemma}]

Let $E = D + D^\top$, and $S = E^\top E$ (which is positive-semidefinite).

It is known that among the matrices $C$ s.t. $C^\top C = S$ (i.e. the ``square roots" of $S$), there is a unique symmetric positive semi-definite root $R$.  

It is also known that the nuclear norm $\|E\|_*$ of $E$, which is defined as the sum of the singular values (denoted $\sigma_i$) of $E$, is equal to the trace of $R$.

By the ``unitary freedom" of matrix square roots 
we know that since $E$ is a square root of $S$, it is related to $R$ by $E = UR$ (i.e. $R = U^\top E$), for some unitary matrix $U$.  And because $U$ is unitary, we have that $|[U]_{i,j}| \leq 1$ for all $i$ and $j$.

Also, because $E$ is a real symmetric matrix, we have that $\sigma_i = |\lambda_i|$, where $\lambda_i$ are its eigenvalues. 

Combining these facts we have
\begin{align*}
\sum_i |\lambda_i| &= \sum_i \sigma_i = \|E\|_* = \trace(R) = \trace(U^\top E) = \sum_{i,j} [U]_{i,j} [E]_{i,j} \\
&\leq | \sum_{i,j} [U]_{i,j} [E]_{i,j} | \leq \sum_{i,j} |[U]_{i,j}| |[E]_{i,j}| \leq \sum_{i,j} |[E]_{i,j}| = 2 \sum_{i,j} |[D]_{i,j}| = 2\Delta
\end{align*}

Now consider the matrix $2I + E$.  The eigenvalues of this matrix, denoted by $\gamma_i$, are given by $\gamma_i = \lambda_i + 2$ for each $i$. The rank of $2I + E$ is given by the number of non-zero eigenvalues it has, or in other words, the number of eigenvalues $\lambda_i$ of $E$ which are not equal to $-2$. Since $\sum_i |\lambda_i| \leq 2\Delta$ we have that at most $\Delta$ of the $\lambda_i$'s can be equal to $-2$, and thus the rank of $E + 2I$ is at least $k - \Delta$.

Then using the fact that $2I + E = (I + D) + (I + D)^\top$ we have that
\begin{align*}
k - \Delta &\leq \rank(2I + E) = \rank( (I + D) + (I + D)^\top ) \\ &\leq \rank( I + D ) + \rank( (I + D)^\top ) = 2 \rank( I+D )
\end{align*}
and thus $\rank( I+D ) \geq k/2 - \Delta/2$.

\end{proof}

\section{Proofs for Section \ref{sec:depth_analysis}}

\begin{proof}[Proof of Lemma \ref{lemma:mult_ident_test}]

Define $q(y) = q_1(y) - q_2(y)$.  Clearly this is a multilinear polynomial as well.  Also observe that $q_1(y) = q_2(y) \iff q(y) \equiv q_1(y) - q_2(y) = 0$.  Thus, in order to prove the lemma it suffices to establish the following proposition.

\begin{proposition}
\label{prop:mult_ident_test}
If $q$ is a multilinear polynomial over $y$ with $q(y) = 0$ for all values of $y$ in $\{0,1\}^\ell$, then $q(y) = 0$ for all values of $y$ in $\Real^\ell$.
\end{proposition}

\begin{proof}[Proof of proposition]
%
%
%
%


The proof will proceed by induction on the number $\ell$ of input variables.

The base case $\ell=0$ holds trivially.

For the inductive case, suppose that $q$ is a multilinear polynomial in $y$ with $q(y) = 0$ for all values of $y$ in $\{0,1\}^\ell$. Because $q$ is multilinear we can write it as
\begin{align}
q(y) = y_\ell g(y_{-\ell})+h(y_{-\ell})
\label{exprforq}
\end{align}
where $g$ and $h$ are multilinear polynomials in $y_{-\ell} \equiv (y_1,...,y_{\ell-1})$. 

Since $q(y)=0$ for all values of $y$ in $\{0,1\}^\ell$, it follows trivially that $h(y_{-\ell})=0$ for all values of $y_{-\ell}$ in $\{0,1\}^{\ell-1}$. Thus, by the inductive hypothesis we have that $h(y_{-\ell})=0$ for all values of $y_{-\ell}$ in $\Real^{\ell-1}$.

Combing this with eqn.~\ref{exprforq} we have that $q(y) = y_{\ell} g(y_{-\ell})$ for all values of $y$ in $\Real^\ell$.

For values of $y$ in $\Real^\ell$ such that $y_\ell = 1$, we have $q(y) = g(y_{-\ell})$.  Then using the fact that $q(y) = 0$ for all values of $y$ in $\{0,1\}^\ell$ (and in particular ones in $\{0,1\}^{\ell-1}\times \{1\}$), it follows that $g(y_{-\ell}) = 0$ for all values of $y_{-\ell}$ in $\{0,1\}^{\ell-1}$.  Again, by the inductive hypothesis we thus have that $g(y_{-\ell})=0$ for all values of $y_{-\ell}$ in $\Real^{\ell-1}$.

Thus we can conclude that $q(y) = y_{\ell} g(y_{-\ell}) = y_{\ell} 0 = 0$ for all values of $y$ in $\Real^\ell$.

\end{proof}

\end{proof}

\begin{proof}[Proof of Theorem \ref{thm:SPN_depth_hierarchy}]

Let $g_{d+1}$ be as in Theorem \ref{thm:raz_depth_hierarchy}.

Then we have that there is a monotone syntactically multilinear arithmetic circuit $\Phi$ over $x$ of size $\bigO(n)$ and product-depth $d+1$ which computes $g_{d+1}$ for all values of $x$ in $\Real^n$.  By the discussion in Section \ref{sec:polys_and_multi} we have that $\Phi$ is equivalent to a decomposable SPN whose univariate functions $f$ are identity functions.  Moreover, by Proposition \ref{prop:completeness_transform} we can transform this decomposable SPN into a decomposable and complete SPN of size $\bigO(n^2)$.  

Now suppose by contradiction that there is a D\&C SPN $\Phi$ of size $n^{o(\log(n)^{1/2d})}$ and product-depth $d$ that computes $g_{d+1}$ for all values of $x$ in $\{0,1\}^n$.  By the discussion in Section \ref{sec:polys_and_multi}, we can transform this into an equivalent syntactically multilinear arithmetic circuit over $x$ by replacing the univariate functions with affine functions which are equivalent on binary-valued $x$'s, and then replacing these affine functions with trivial subcircuits that compute them.  Because this only adds a layer of $\bigO(s)$ sum nodes, where $s$ is the size of $\Phi$, the size of the circuit remains $n^{o(\log(n)^{1/2d})}$ and product-depth remains $d$.  And because both $g_{d+1}$ and the output function of this circuit are multilinear polynomials, by Lemma \ref{lemma:mult_ident_test} (with $y = x$) we have that the circuit's output polynomial agrees with $g_{d+1}$ for all values of $x$ in $\Real^n$.  This contradicts our choice of $g_{d+1}$.

\end{proof}

\section{Proofs for Section \ref{sec:formula_limitation}}

\begin{proof}[Proof of Theorem \ref{thm:formula_limitation}]

Let $g$ be as in Theorem \ref{thm:raz_formula_limitation}.

Then we have that there is a monotone syntactically multilinear arithmetic circuit $\Phi$ over $x$, of size $\bigO(n)$ and nodes of maximum in-degree $\bigO(n^{1/3})$, which computes $g$ for all values of $x$ in $\Real^n$.  By the discussion in Section \ref{sec:polys_and_multi} we have that $\Phi$ is equivalent to a decomposable SPN whose univariate functions $f$ are identity functions.  Moreover, by Proposition \ref{prop:completeness_transform} we can transform this decomposable SPN into a decomposable and complete SPN of size $\bigO(n + n + n n^{1/3}) = \bigO(n^{4/3})$. 

Now suppose by contradiction that there is a D\&C SPN formula $\Phi$ of size $n^{o(\log(n))}$ that computes $g$ for all values of $x$ in $\{0,1\}^n$.  By the discussion in Section \ref{sec:polys_and_multi}, we can transform this into an equivalent syntactically multilinear arithmetic circuit over $x$ by replacing the univariate functions with affine functions which agree on all binary-valued inputs, and then replacing these with trivial subcircuits that compute them.  Because the subcircuits we add are formulas, and do not share nodes with the existing circuit or with each other, this construction clearly gives a formula. Moreover, the size of the circuit remains $n^{o(\log(n))}$.  And because both $g$ and the output polynomial of this formula are multilinear polynomials, we have by Lemma \ref{lemma:mult_ident_test} (with $y = x$) that the formula's output polynomial agrees with $g$ for all values of $x$ in $\Real^n$.  This contradicts our choice of $g$.

\end{proof}

\section{Proofs for Section \ref{sec:spanning_tree_lb}}

\begin{proof}[Proof of Theorem \ref{thm:efficientmarginalcomputation}]

The algorithm proceeds as follows.

We first construct the subgraph $H$ of $K_m$ whose edges are given precisely by those elements of $i \in I$ for which $x_i = 1$.  In other words, $H$ is the minimal subgraph of $K_m$ which is consistent with the values of $x_I$.

Next, we verify that $H$ is acyclic, which can be done in $\bigO(m^2)$ time using a depth first search.  

If $H$ is not acyclic, then the algorithm returns $0$, as any $G$ which is consistent with the given values of $x_I$ must have $H$ as a subgraph, and thus cannot be a tree.

Otherwise, $H$ is a forest (i.e. each connected component is a tree, including lone vertices with no incident edges) and we proceed as follows. 

We construct a new edge-labeled multi-graph $M$ where each vertex of $M$ corresponds to a connected component of $H$.  For each edge with a label in $\{1,...,n\} \setminus I$ (i.e. those edges $i$ for which $x_i$ is not given a fixed value), which will be between vertices in distinct connected components of $H$, we add an edge (with the same label) between the corresponding vertices of $M$.

We now claim that the number of (edge labeled) spanning trees of $M$ is equal to the number of spanning trees of $K_m$ consistent with the values of $x_I$.  To establish this, we will exhibit a bijection between the two sets.

Consider a spanning tree $T'$ of $M$.  We can construct a subgraph $T$ of $K_m$ by taking $H$ and adding to it the edges of $K_m$ corresponding to labels of the edges of $T'$.  Since the labels of the edges of $T'$ are disjoint from $I$ by construction, and $H$ is clearly consistent with the values of $x_I$, this extended graph will be as well. It is also not hard to see that it is a spanning tree of $K_m$, and that this mapping is 1-1.

Now consider a spanning tree $T$ of $K_m$ consistent with the values of $x_I$.  We can construct a subgraph $T'$ of $M$ whose set of edges are given precisely by $S \setminus I$ where $S$ is the set of labels of edges in $T$.  Because $T$ has $H$ as a subgraph, it clearly cannot contain any edges between vertices in the same connected component of $H$ (as this would cause a cycle), and so $S \setminus I$ indeed contains only labels of edges in $M$.  Again, it is not hard see that $T'$ is a spanning tree of $M$, and that this mapping is 1-1.

Thus it remains to count the (edge-labeled) spanning trees of $M$.

To do this, we form the generalized Laplacian matrix $L$ of $M$ given by
\begin{align*}
L_{u,v} =
\begin{cases}\deg(v) & \mbox{if}\ u = v \\
-j & \mbox{if}\ u \neq v\ \mbox{where}\ j \mbox{ is the number of edges between } u \mbox{ and } v
	\end{cases}
\end{align*}
and compute one of its co-factors.  For example, we can compute the determinant of $L$ with the first row and column removed.  The asymptotic cost of this is, by the results of \citet{bunch}, are the same as the asymptotic cost of matrix multiplication, which is known to be at worst $\bigO(m^{2.38}) \approx \bigO(n^{1.19})$ \citep{coppersmith}.

By a generalization of the Matrix Tree Theorem \citep{tutte2001graph}, the value of any of the cofactors of $L$ is equal to the number of spanning trees of $M$.

\end{proof}

\begin{proof}[Proof of Theorem \ref{thm:decomposition}]

Let $\Phi$ be a D\&C SPN of size $s$.

We first transform $\Phi$ into an equivalent D\&C SPN where the maximum fan-in of each product node is 2.  This is accomplished by replacing each product node of fan-in $\ell \geq 2$ with a subcircuit structured like a binary tree whose total size is $1 + 2 + 4 + ... + \ell/2 = \ell-1 \leq s$.  The size of the new $\Phi$ will thus be $t \leq s^2$.

Next, we decompose $\Phi$ according to the following iterative procedure.

Starting with $\Phi_0 = \Phi$, at each stage $i$, we find a node $v_i$ in $\Phi_{i-1}$ whose dependency-scope $x_{v_i}$ satisfies $n/3 \leq |x_{v_i}| \leq 2n/3$ (we will show later why this is always possible).  We then remove $v_i$ from $\Phi_{i-1}$ (effectively just replacing it with 0), and pruning any orphaned children that result, noting that the resulting circuit $\Phi_i$ is also a decomposable and complete SPN over the same input variables ($x$), or is the empty circuit.

The procedure stops at the step $k$, when $\Phi_k$ becomes the empty circuit (so that $q_{\Phi_k}$ is the zero polynomial).  Clearly, $k \leq t$, since $\Phi_0$ starts with $t$ nodes, and each step of the procedure removes at least one node.

We can express the output of $q_{\Phi}$ as a telescoping sum of differences as follows
\begin{align}
\label{eqn:telescope}
q_{\Phi} = (q_{\Phi_0} - q_{\Phi_1}) + (q_{\Phi_1} - q_{\Phi_2}) + ... + (q_{\Phi_{k-1}} - q_{\Phi_k}) + q_{\Phi_k}
\end{align}

Suppose that $v_i$ is the node removed at stage $i$, and consider the SPN $\Phi'_{i-1}$ obtained by replacing $v_i$ in $\Phi_{i-1}$ with an input node labeled with a new variable $z_i$.  If we define the dependency-scope $x_{z_i}$ of $z_i$ to be $x_{v_i}$ (thinking of it as a `function' which depends on $x_{v_i}$ analogously to how each of the elements of $f$ are functions which depend on one of the $x_j$'s) then clearly $\Phi'_{i-1}$ is an SPN and inherits decomposability and completeness from $\Phi_{i-1}$.  Moreover, since $q_{\Phi_i}$ is clearly obtained by setting $z_i = 0$ in $q_{\Phi'_{i-1}}$, we have that $q_{\Phi'_{i-1}} - q_{\Phi_i}$ is a polynomial in $z_i$ and $f$ consisting precisely of those monomials from $q_{\Phi'_{i-1}}$ which have $z_i$ as a factor, with the same corresponding coefficients (which are all non-negative).  

Because $\Phi'_{i-1}$ is a D\&C SPN, we have from Theorem \ref{thm:3way_equiv} that each monomial $m$ in $q_{\Phi'_{i-1}}$ with $z_i$ as a factor is of the form $z m'$ where $x_{m'} \cap x_{z_i} = \emptyset$.  Thus $q_{\Phi'_{i-1}} - q_{\Phi_i} = z_i \psi_i$ for some non-negative polynomial $\psi_i$ over $f$ with $x_{\psi_i} \cap x_{z_i} = \emptyset$ and $x_{\psi_i} \cup x_{z_i} = x$.

Substituting $z_i = q_{v_i}$ in $q_{\Phi'_{i-1}}$ recovers $q_{\Phi_{i-1}}$, and thus it follows that $q_{\Phi_{i-1}} - q_{\Phi_i} = q_{v_i} \psi_i$. 

Define $g_i = q_{v_i}$ and $h_i = \psi_i$.  Then we can rewrite eqn.~\ref{eqn:telescope} as $q_{\Phi} = \sum_{i=1}^k g_i h_i$.  Noting that $v_i$ was chosen so that $n/3 \leq |x_{v_i}| \leq 2n/3$, this gives the result.

It remains to show that at each stage $i$ we can find a node $v_i$ in $\Phi_{i-1}$ satisfying $n/3 \leq |x_{v_i}| \leq 2n/3$.  To do this we perform the following procedure.  Starting with the root node of $\Phi_{i-1}$ we procede down along the circuit towards the input nodes, following a path given by always taking the child with the largest-sized dependency-scope among the children of the current node.  We follow this path until we arrive at a node which satisfies the required properties.  To see that we eventually do find such a node, note first that if the current node $u$ is a sum node, the dependency-scope of any of its children will be the same as its own (due to completeness), and if $u$ is a product node, then since the number of children is $\leq 2$, at least one child must have an dependency-scope whose size is at least half that of $u$'s dependency-scope (due to decomposability).  If the size of the current node's dependency-scope is $> 2n/3$ it must be the case that at least one child has an dependency-scope of size $\geq n/3$.  Because the size of the dependency-scope never increases along this path, and must eventually become 1 (at an input node), it thus follows that the first node on the path whose dependency-scope is of size $\leq 2n/3$ will satisfy the required properties.

\end{proof}

\begin{proof}[Proof of Theorem \ref{thm:sequenceofspns}]

Let $\{\Phi_j\}_{j=1}^{\infty}$ be such a sequence.

Applying Theorem \ref{thm:decomposition} to each $\Phi_j$ we have that
\begin{align}
\label{sumofproducts}
q_{\Phi_j}(x) = \sum_{i=1}^{k_j} g_{i,j} h_{i,j}
\end{align}
where for each $j$, $k_j \leq s^2$ and the $g_{i,j}$'s and $h_{i,j}$'s are polynomials in the different $f$'s (and thus functions of $x$) satisfying the conditions in Theorem \ref{thm:decomposition} (eqn.~\ref{comparablescopeconditions}).

Note that since there are finitely many integers $\geq s^2$ and finitely many possible choices of the dependency-scopes of $g_{1,j},...,g_{k_j,j}$ and $h_{1,j},...,h_{k_j,j}$ for any particular fixed $j$, it follows that at least one joint choice of these repeats infinitely often in the sequence, and thus we can replace $\{\Phi_j\}_{j=1}^{\infty}$ with a subsequence where $k_j$ is constant and independent of $j$, and the dependency-scopes of the corresponding $g_{i,j}$ and $h_{i,j}$'s also do not depend on $j$.  We will refer to former quantity simply as $k$ (without the subscript), and the constant dependency-scopes of $g_{i,j}$ and $h_{i,j}$ as $y_i$ and $z_i$ (resp.).

For notational convenience we will redefine $\{\Phi_j\}_{j=1}^{\infty}$ (and all quantities which are derived from it, such as the $g_{i,j}$'s and $h_{i,j}$'s) to be this subsequence, as we will do going forward whenever we talk about ``replacing" the current $\{\Phi_j\}_{j=1}^{\infty}$ with a subsequence.

For all values of $x$, we have that each term in the sum $q_{\Phi_j}(x) = \sum_{i=1}^{k} g_{i,j}(y_i) h_{i,j}(z_i)$ is non-negative, and thus $q_{\Phi_j}(x) \geq g_{i,j}(y_i) h_{i,j}(z_i) \geq 0$.  And since $q_{\Phi_j}(x)$ converges to $\gamma(x)$ as $j \to \infty$, it then follows that the sequence of real-values $\{g_{i,j}(y_i) h_{i,j}(z_i)\}_{j=1}^\infty$ is bounded.

Define $\nu_j \in \Real^{dk}$ to be the $dk$ dimensional vector consisting of  $g_{i,j}(y_i) h_{i,j}(z_i)$ for each $i \in [k]$ and each of the $d$ possible values of $x$.  By the Bolzano-Weierstrass theorem, since $\{\nu_j\}_{j=1}^\infty$ is a bounded sequence (each of its component sequences are bounded, and it has finitely many components), it has a convergent subsequence.  Replace $\{\Phi_j\}_{j=1}^{\infty}$ with the corresponding subsequence (i.e. the subsequence given by the corresponding $j$'s).

Now given that the new $\{\nu_j\}_{j=1}^\infty$ converges it thus follows by definition of $\nu_j$ that for each value of $x$, $\{ g_{i,j}(y_i) h_{i,j}(z_i) \}_{j=1}^\infty$ converges to some value which we will denote $\eta_i(x)$, where we note that $\sum_{i=1}^k \eta_i(x) = \gamma(x)$.

By Lemma \ref{lemma:factor_sequence} (below) it follows that each $\eta_i$ can be written as $g_i h_i$ for functions $g_i$ and $h_i$ of $y_i$ and $z_i$ (resp.).  Thus we have that $\gamma(x) = \sum_{i=1}^k \eta_i(x) = \sum_{i=1}^k g_i(y_i) h_i(z_i)$ for each value of $x$, which completes the proof.

\end{proof}

\begin{lemma}
\label{lemma:factor_sequence}
Suppose $\{g_j\}_{j=1}^\infty$ and $\{h_j\}_{j=1}^\infty$ are sequences of real-valued functions of $y$ and $z$ (resp.), where $y$ and $z$ are disjoint subsets/tuples of $x$, and that the size of the range of possible values of $x$ is given by some $d < \infty$, and finally that $\{g_j h_j\}_{j=1}^\infty$ converges pointwise to some function $\eta$ of $x$.  Then there exists functions $g$ and $h$ of $y$ and $z$ (resp.) such that $\eta = gh$.
\end{lemma}

\begin{proof}[Proof of Lemma \ref{lemma:factor_sequence}]

Consider a subsequence of $\{g_j\}_{j=1}^\infty$ where each $g_j$ is not the zero function, and replace $\{g_j\}_{j=1}^\infty$ with this subsequence, and $\{h_j\}_{j=1}^\infty$ with its corresponding subsequence (according to the index $j$).  Note that if such a subsequence of $\{g_j\}_{j=1}^\infty$ doesn't exist, then after some point in the sequence, $g_j$ is always the zero function and thus it follows that $\eta$ is the zero function, and we can take both $g$ and $h$ to be the zero functions.

Define $\alpha_j = \max_y |g_j(y)|$ for each $j$.  We note that for all $j$, $\alpha_j \neq 0$ by how the sequence $\{g_j\}_1^\infty$ was sub-selected above.

Now define a pair of new sequences of functions by
\begin{align*}
g'_j(y) &= \frac{g_j(y)}{\alpha_j} \\
h'_j(z) &= \alpha_j h_j(z)
\end{align*}

For all values of $x$, we have that since $g_j(y) h_j(z) = g'_j(y) h'_j(z)$, $\{g'_j(y) h'_j(z)\}_{j=1}^\infty$ thus converges to $\eta(x)$.

Let $\lambda_j$ denote the finite dimensional vector formed by evaluating $g'_j$ for every possible value of $y$ (there are $\leq d$ of these, since there are $d$ possible values of $x$).  Note that construction, $|g_j(y)| \leq 1$ for all values of $y$, and thus the sequence $\{\lambda_j\}_{j=1}^\infty$ is bounded.  By the Bolzano-Weierstrass theorem it has convergent subsequence.  Replace $\{g'_j\}_{j=1}^{\infty}$ and $\{h'_j\}_{j=1}^{\infty}$ with the corresponding subsequences (i.e. the subsequence given by the corresponding $j$'s).

Now we have that $\{g'_j\}_{j=1}^{\infty}$ converges point-wise to some function $g$ of $y$.  And it remains to find the promised $h$.

Note that since there are only finitely many possible values of $y$, there is some particular value of $y'$ of $y$, for which $g_j'(y') = \beta$ for infinitely many $j$'s, for some $\beta \neq 0$.  To see this, note that we can pick $y'$ so that $|g_j'(y)|$ is maximized infinitely often, and so that we have that either $g_j'(y') = 1$ infinitely often, or $g_j'(y') = -1$ infinitely often.  Replace $\{g'_j\}_{j=1}^{\infty}$ with this subsequence (and $\{h'_j\}_{j=1}^{\infty}$ with its corresponding subsequence).

Note that since $\{g'_j(y) h'_j(z)\}_{j=1}^\infty$ converges for all values of $x$, and that the variables in $y$ are disjoint from those in $z$, it is certainly true that $\{g'_j(y') h'_j(z)\}_{j=1}^\infty$ converges for all values of $z$.  But this is equal to $\{\beta h'_j(z)\}_{j=1}^\infty$, and since $\beta \neq 0$ it follows that $\{h'_j(z)\}_{j=1}^\infty$ converges to some $h(z)$ for all values of $z$. 

Since both $g'_j$ and $h'_j$ individually converge point-wise to $g$ and $h$ (resp.), and we know that $\{g'_j h'_j\}_{j=1}^\infty$ converges point-wise to $\eta$, it follows that $\eta = gh$.

\end{proof}

\begin{proof}[Proof of Proposition \ref{prop:dichotomy}]

We can assume without loss of generality that $a$ and $b$ are red, and $c$ is blue (since $g$ and $h$ are interchangeable).

Suppose by contradiction that there are a pair of values $x'$ and $x''$ of $x$ s.t. $G_{x'}$ contains both $a$ and $b$, and $G_{x''}$ contains $c$, and $g(y')h(z') > 0$ and $g(y'')h(z'') > 0$, where $(y', z')$, and $(y'', z'')$ are the values of $(y,z)$ corresponding to $x'$ and $x''$ respectively.

Let $x'''$ denote a value of $x$ which agrees with $y'$ on $y$ and $z''$ on $z$.  Then clearly $G_{x'''}$ contains all three edges of the constraint triangle and so cannot be a spanning tree, implying that $d(x''') = 0$.  

But we have that both $g(y') > 0$ and $h(z'') > 0$ (which follow from $g(y')h(z') > 0$ and $g(y'')h(z'') > 0$ respectively), and so $d(x''') > 0$, which is a contradiction.

\end{proof}

\begin{proof}[Proof of Lemma \ref{lemma:dichromatictriangles}]

Our strategy for lower bounding the number of dichromatic triangles will be to instead upper bound the total number of monochromatic triangles of $K_m$.

We will make use of a result of \citet{fisher1989lower} which says that given an arbitrary graph $G$ with $e$ edges, the number of triangles in $G$ is upper bounded by
\begin{align*}
\frac{\left(\sqrt{8e + 1} - 3\right)e}{6} \leq \frac{\sqrt{2}}{3} e^{3/2}
\end{align*}

Let $G_1$ be the subgraph of $K_m$ formed by taking only the red edges, and $G_2$ the subgraph of $K_m$ formed by taking only the blue ones.  Clearly the number of monochromatic red (or blue) triangles in the original graph is just the number of total triangles in $G_1$ (or $G_2$).

Applying the above upper bound to both $G_1$ and $G_2$ it thus follows that the total number of monochromatic triangles of the original colored graph is upper bounded by
\begin{align*}
\frac{\sqrt{2}}{3} r^{3/2} + \frac{\sqrt{2}}{3} (n-r)^{3/2}
\end{align*}

For $r$ satisfying $n/3 \leq r \leq 2n/3$ the function attains its maximum value at both $r = n/3$ and $r = 2n/3$, and is given by $\sqrt{2}/3 ( (1/3)^{3/2} + (2/3)^{3/2} ) n^{3/2} \leq 0.74 \sqrt{2}/3 \: n^{3/2}$ (by calculation).  Using $n = \binom{m}{2} \leq m^2 / 2$, this upper bound can be written as $0.74 \frac{\sqrt{2}}{3} m^3 / 2^{3/2} = 0.74 m^3/6$.

There are $\binom{m}{3}$ total triangles in $K_m$.  For $m \geq 20$, it is straightforward to verify that $\binom{m}{3} \geq 0.84 m^3/6$.

By subtracting the upper bound on the number of monochromatic triangles from the lower bound on the number of total triangles we arrive at the following lower bound on the total number of dichromatic triangles
\begin{align*}
(0.84 - 0.74) m^3/6 = m^3/60
\end{align*}

\end{proof}

\begin{proof}[Proof of Lemma \ref{lemma:fractionoftrees}]

Note that the proportion of spanning trees of $K_m$ which obey all of the constraints can be interpreted as the probability that a sample from uniform distribution over all spanning trees of $K_m$ (i.e. the distribution $\mathcal{D}$) obeys all of the constraints.  We will upper bound this probability by analyzing the behavior of a simple algorithm due to \citet{aldous1990random} which samples from $\mathcal{D}$.

The algorithm is described as follows.  Starting with an empty graph $T$ and a uniformly sampled initial choice for the ``current vertex" $v$, it iterates the following steps.  Uniformly sample one of the vertices $u$ of those adjacent to $v$ in $G$, and make this the current vertex. If $u$ has not been previously visited, add the edge $(v,u)$ to $T$.  

The algorithm terminates once every vertex has been visited at least once.

For ease of exposition we will perform a minor modification to this algorithm by allowing the current vertex $v$ to be re-sampled (with the same probability as sampling any one of the adjacent vertices).  Clearly this doesn't change the distribution sampled by the algorithm.

Consider applying this algorithm with $G = K_m$.  In this case it becomes particularly simple, since at each step it samples the next vertex uniformly at random from the set of all $m$ vertices, thus performing a uniform random walk on the graph.

For the sake of simplicity we will consider the sequence of vertices produced by this random walk in consecutive ``stages" of $3$ at a time.  Note that the 3 vertices sampled at any given stage are done so uniformly from the set of all vertex triples (allowing repeats).

We will encode constraints as ordered triples of vertices, which we will call ``constraint triples".  Each constraint of the first form is represented by a pair of constraint triples $(u,v,w)$ and $(w,v,u)$, where $a = (u,v)$ and $b = (v,w)$.  And for convenience, we encode constraints of the second form similarly as triples $(u,v,w)$ and $(w,v,u)$ where $c = (u,v)$ and $w \neq u,v$ can be any other vertex from $K_m$.  This gives $2C$ total distinct constraint triples.

Note that if during the random walk on $K_m$ the algorithm visits the vertices $u$, $v$ and $w$ in consecutive order, and $v$ and $w$ have not yet been visited before, both the edges $(u,v)$ and $(v,w)$ will be added to $T$.  Thus, if such a $u$, $v$ and $w$ are encountered where $(u,v,w)$ is a constraint triple, it is easy to see that the final $T$ will be in violation of the corresponding constraint.

At any given stage, the number of triples of the form $(u,v,w)$, where either $v$ or $w$ has been previously visited, is $m(m-\ell)\ell + m\ell(m-\ell) + m\ell^2 = \ell m (2m - \ell)$, where $\ell$ is the total number of vertices previously visited.  By the beginning of the $(i+1)$-th stage we note that the algorithm has visited at most $3i$ vertices, and thus $\ell \leq 3i$.  From this it follows that the total number of such triples is at most $3im(2m - 3i)$.

Thus, of the $m^3$ possible triples that can be sampled at stage $i+1$, there at least $2C - 3im(2m - 3i)$ of them which correspond to constraint triples of the form $(u,v,w)$ where neither $v$ nor $w$ has been visited at a previous stage.

Conditioned on \emph{any} previous choices made by the algorithm, the probability that the 3 vertices picked at stage $i+1$ imply a constraint violation is lower bounded by the worst-case probability that they correspond to one of the remaining constraint triples.  This is
\begin{align*}
\frac{2C - 3im(2m - 3i)}{m^3} = \frac{2C - 6im^2 + 9i^2m}{m^3} \geq \frac{2C - 6im^2}{m^3}
\end{align*}
and so the probability of no such violation occurring at stage $i+1$, conditioned on any previous choices made by the algorithm, is upper bounded by
\begin{align*}
1 - \frac{2C - 6im^2}{m^3}
\end{align*}

The probability of no such constraint violation occurring during the entire run of the algorithm is upper bounded by the probability of no such violation occurring by stage $t$, which is itself upper bounded by
\begin{align*}
\prod_{i=1}^t \left( 1 - \frac{2C - 6im^2}{m^3} \right ) \leq \left( 1 - \frac{2C - 6t m^2}{m^3} \right )^t
\end{align*}

Plugging in $t = C/(6m^2)$ this becomes
\begin{align*}
\left( 1 - \frac{C}{m^3} \right )^{C/(6m^2)}
\end{align*}

\end{proof}

\end{document}